\documentclass{article} 
\usepackage{iclr2027_conference,times}

\usepackage[utf8]{inputenc} 
\usepackage[T1]{fontenc}    
\usepackage{hyperref}       
\usepackage{url}            
\usepackage{booktabs}       
\usepackage{amsfonts}       
\usepackage{nicefrac}       
\usepackage{microtype}      
\usepackage{xcolor}         
\usepackage{siunitx}
\usepackage{tabularx}
\usepackage{threeparttable}
\usepackage{amsmath}
\usepackage{amssymb}
\usepackage{mathtools}
\usepackage{amsthm}
\usepackage{bm}
\usepackage{algorithm}
\usepackage{algpseudocodex}
\usepackage{multirow}
\usepackage{enumitem}
\usepackage{wrapfig}
\usepackage{minitoc}
\usepackage{xspace}

\usepackage{graphicx}
\usepackage{tikz}
\usepackage[percent]{overpic}
\usetikzlibrary{arrows.meta,calc,shadings}

\usepackage[capitalize,noabbrev]{cleveref}

\theoremstyle{plain}
\newtheorem{theorem}{Theorem}[section]

\newtheorem{lemma}[theorem]{Lemma}

\theoremstyle{definition}

\newtheorem{example}{Example}
\theoremstyle{remark}
\newtheorem{remark}[theorem]{Remark}

\usepackage[textsize=tiny]{todonotes}

\definecolor{mynavy}{HTML}{1D6996}

\newcommand{\condon}{\,|\,}

\newcommand{\bx}{\bm{x}}
\newcommand{\by}{\bm{y}}
\newcommand{\bu}{\bm{u}}
\newcommand{\bv}{\bm{v}}
\newcommand{\ba}{\bm{a}}
\newcommand{\bs}{\bm{s}}
\newcommand{\bA}{\bm{A}}
\newcommand{\bo}{\bm{o}}
\newcommand{\der}{\mathrm{d}}
\newcommand{\R}{\mathbb{R}}
\newcommand{\E}{\mathbb{E}}
\newcommand{\D}{\mathcal{D}}

\newcommand{\bth}{\bm{\theta}}
\newcommand{\bph}{\bm{\phi}}
\newcommand{\kl}{D_{\text{KL}}}
\newcommand{\gaussian}{\mathcal{N}(\bm{0},\bm{I})}

\newcommand{\method}{SAVE\xspace}

\newcommand{\pusht}{\texttt{Push-T}\xspace}
\newcommand{\libero}{\texttt{LIBERO}\xspace}
\newcommand{\liberoten}{\texttt{LIBERO-Long}\xspace}
\newcommand{\liberoplus}{\texttt{LIBERO-Plus}\xspace}

\hypersetup{
    colorlinks=true,    
    linkcolor=mynavy,     
    citecolor=mynavy,     
    filecolor=black,     
    urlcolor=mynavy,      
    pdfpagemode=FullScreen,
}

\newif\ifpaper
\newif\ifappendix

\papertrue
\appendixtrue

\title{Uncertainty Quantification for Flow-Based Generalist Robot Policies}

\iclrfinalcopy 
\begin{document}

\doparttoc 
\faketableofcontents 

\author{%
\mbox{%
Ralf R\"omer$^{1}$\hspace{2em}%
Maximilian Seeliger$^{2}$\hspace{2em}%
Saida Liu$^{1}$\hspace{2em}%
Ben Sturgis$^{1}$}\AND%
\mbox{%
Marco Bagatella$^{2,3}$\hspace{2em}%
Daniel Marta$^{2}$\hspace{2em}%
Andreas Krause$^{2}$\hspace{2em}%
Angela P. Schoellig$^{1,4}$} \\
\\
\small
$^{1}$Technical University of Munich \hspace{3em} 
$^{2}$ETH Zurich \hspace{3em} 
$^{3}$MPI for Intelligent Systems, Tübingen \\
\small
$^{4}$Robotics Institute Germany
}

\maketitle

\ifpaper

\begin{abstract}
Generalist robot policies, such as vision-language-action models~(VLAs) and world-action models~(WAMs), combine powerful pretrained backbones with expressive generative action heads trained via flow matching on large-scale robotic datasets.
Despite their strong empirical performance in robotic manipulation, these policies lack mechanisms to quantify confidence in their predictions and to detect when their actions may be unreliable.
This presents a critical limitation for real-world deployment in non-stationary environments, where models inevitably encounter scenarios outside their pretraining distribution and may fail without warning.
To address this, we derive an efficient method to quantify epistemic uncertainty in flow-matching models by leveraging velocity-field disagreement~(VFD) across a small ensemble.
We successfully use this uncertainty estimate for detecting failures during deployment and active fine-tuning of flow-based generalist policies. 
For the latter, we propose \method, a simple yet effective method for uncertainty-guided active multitask fine-tuning that reduces the number of costly expert demonstrations required to adapt generalist policies to new tasks. 
We conduct experiments in simulation and the real world, across VLAs and a WAM. VFD yields better-calibrated uncertainty estimates predictive of downstream performance and detects failures with~\qty {8}{}\,pp higher overall accuracy than existing methods.
Across three real-world tasks, SAVE improves final average success from~\qty{39}{\percent} to~\qty{47}{\percent} with a fixed demonstration budget.
Our results show that measuring epistemic uncertainty with VFD enhances both failure awareness and adaptation of generalist robot policies.
Project website: \href{https://tum-lsy.github.io/uq_generalist_policies/}{tum-lsy.github.io/uq\_generalist\_policies}.
\end{abstract}


\section{Introduction}
The premise of generalist robot policies~\citep{black2024pi0, kim2025cosmospolicy} is to bring the large-scale pre-training paradigm that has been successful in deep learning in recent years~\citep{dosovitskiy2020image, awais2025foundation} into the physical world.
Driven by the increasing availability of large-scale robotics datasets~\citep{o2024open}, generalist policies have demonstrated impressive performance in long-horizon tasks, exhibiting remarkable robustness and emergent zero-shot capabilities~\citep{physical2025pi0_5}. 
Architecturally, the field has largely converged on combining a pre-trained vision-language or video model backbone with an action-generation diffusion transformer trained via flow matching~(FM)~\citep{lipman2023flow} to generate continuous robot actions.
Flow-based generalist policies provide stable training, relatively fast inference, and the expressivity needed to model the complex, highly multimodal distributions of real-world physical demonstrations.

While generalist policies are pre-trained on large-scale static datasets spanning a broad distribution of objects and tasks~\citep{o2024open, bu2025agibot}, deploying these policies in the real world inevitably exposes them to non-stationarity. Users repurpose robots for novel tasks, object appearances may change, and environments evolve.
In such scenarios, pre-trained policies may perform well on some tasks but fail completely on others, without explicit knowledge of which tasks are fully understood from pre-training and which require additional data. 
The inability of these models to communicate \textit{what they do not know} raises two key problems. 
First, generalist policies tend to confidently execute erratic actions in out-of-distribution scenarios~\citep{gu2025safe} rather than abstaining or asking for help when uncertain~\citep{ren2023robots}, raising significant safety concerns~\citep{li2026vision, gu2025safe}.
Second, because these policies cannot inherently identify their knowledge gaps, adapting them to new domains currently requires collecting 
large numbers of human expert demonstrations~\citep{bjorck2025gr00t}, which is prohibitively expensive. 


To address these problems, our work focuses on uncertainty quantification~(UQ) for flow-based generalist policies. Leveraging the generative formulation of these models, we propose an effective method for estimating epistemic uncertainty by measuring divergence between ensembled velocity fields.
Building on this estimator, we propose
\underline{s}ample-efficient \underline{a}ctive fine-tuning via \underline{v}elocity-field \underline{e}pistemic uncertainty~(\method) for flow-based generalist policies.
Our main contributions are threefold: \looseness -1
\begin{itemize}[
    leftmargin=12pt,
    nosep,
    itemsep=2pt,
    parsep=0pt
]
    \item We propose velocity-field disagreement (VFD), a tractable epistemic uncertainty score for FM models, based on a velocity-field expression for the KL divergence between two FM models. 
    \item We propose \method, a framework for uncertainty-guided active fine-tuning of flow-based generalist policies that uses VFD to prioritize tasks and initial states for collecting expert demonstrations.
    \item We empirically evaluate our methods in simulation and the real world, demonstrating that VFD predicts task performance more reliably and detects deployment failures more accurately than baselines, and that \method reduces the amount of expert data required for multitask adaptation.
\end{itemize}


\begin{figure}[t!]
    \centering
    \includegraphics[width=\linewidth]{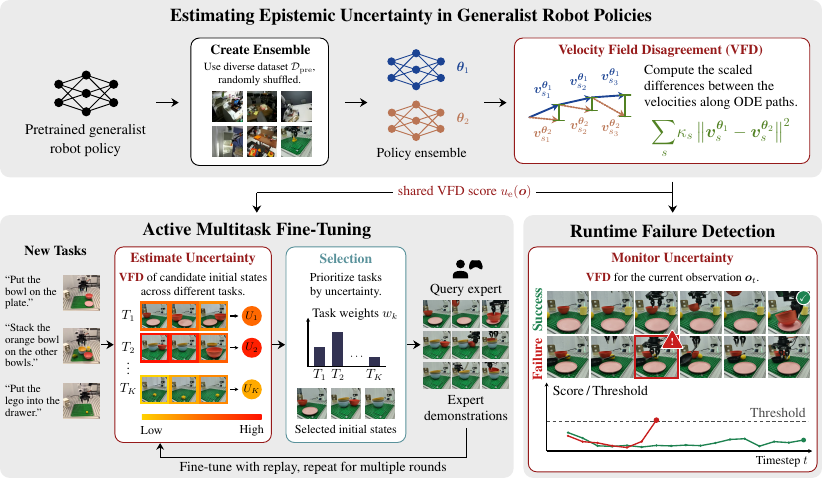} 
    \caption{
    \textbf{Top:} We measure epistemic uncertainty in flow-based generalist policies by measuring scaled differences between ensembled velocity fields.
    \textbf{Bottom:} 
    VFD enhances the adaptability and reliability of generalist policies in multiple ways. Left: We propose \method, an active fine-tuning method that allocates expert demonstration budget to tasks and initial states with the highest uncertainty.
    Right: VFD is an effective runtime monitor for detecting failures during deployment.
    }
    \label{fig:overview}
\end{figure}



\section{Related Work}
\label{sec:related}

\paragraph{Generalist Robot Policies.}
Generalist policies have emerged as a dominant approach to learning complex manipulation tasks by fine-tuning pre-trained vision-language or video backbones on robotic data~\citep{kim24openvla, black2024pi0, kim2025cosmospolicy, yuan2026fastwam}. 
Earlier vision-language-action models~(VLAs) map visual and language observations to discrete or continuous action tokens via regression or classification heads~\citep{kim24openvla, pmlr-v229-zitkovich23a}, but the state of the art has adopted FM as an expressive framework for modeling the highly multimodal action distributions in large-scale datasets~\citep{black2024pi0, fang2026molmoact2}. 
A related concept, world-action models~(WAMs)~\citep{kim2025cosmospolicy, yuan2026fastwam}, leverages the spatiotemporal priors of video models to enhance action generation, typically using latent flow matching.
Despite their expressiveness, generalist policies do not natively quantify uncertainty and thus cannot identify situations in which they lack knowledge of which action to perform~\citep{agia2024unpacking, lee2024diff}, preventing safe, autonomous real-world deployment.
This motivates estimating \textit{where} the model is most uncertain: before deployment, to prioritize situations for querying an expert, and during deployment, to detect failures.

\paragraph{Uncertainty Quantification for Generative Models.} \label{sec:uq_generative}
Uncertainty in deep learning is typically decomposed into epistemic (model) and aleatoric (data) uncertainty~\citep{hullermeier2021aleatoric}.
While the latter is an irreducible property of the data generation process, the epistemic component can be reduced by collecting new and, importantly, the \textit{right data}.
Depending on the model class, UQ has been instantiated in many different ways. Test-time dropout~\citep{loquercio2020general} and inference on augmented inputs~\citep{ayhan2018test} may, for instance, be applied directly at test time and are mostly architecture-agnostic. More powerful approaches, such as deep ensembles~\citep{lakshminarayanan2017simple} and Laplace approximations~\citep{daxberger2021laplace}, require training multiple models or modifying the training procedure.
With the rise of modern generative models~\citep{ho2020denoising, lipman2023flow}, which, despite their expressivity, often remain overconfident~\citep{karczewski2025diffusion}, several recent works have investigated UQ for this model class.
These established approaches have recently been extended to modern generative models such as diffusion and flow matching~\citep{franchi2025towards}, which often remain overconfident~\cite{nalisnick2018deep, karczewski2025diffusion}.
Among recent approaches, \citet{jazbec2025generative} apply a last-layer Laplace approximation to diffusion models and estimate uncertainty in a semantic representation space using a Gaussian approximation to the posterior predictive, which may, however, be unsuitable for multi-modal distributions.
DECU~\citep{berry2024shedding} trains an ensemble of latent diffusion models and performs pairwise distance estimation using the denoising means. 
Other recent works rely on similarity metrics in language space~\citep{franchi2025towards}, leverage hypernetworks~\citep{chan2024estimating}, or train explicit confidence predictors~\citep{mei2025world}. 
Closer to our setting, \citet{ju2026epistemic} estimate epistemic uncertainty for pre-trained VLMs through a Riemannian formulation of FM, but do not target action generation.
UQ is also an active area of research in autoregressive models~\citep{shorinwa2025survey}, with domain-specific techniques spanning from differential entropy over vocabulary~\citep{ling2024uncertainty}, to perplexity~\citep{fadeeva2024fact}, and Dirichlet evidence~\citep{ma2025estimating}, but these methods primarily measure aleatoric uncertainty, i.e., data ambiguity.

\paragraph{Active Learning.}\label{sec:al_generative}
A direct application of reliable uncertainty estimates is in active data selection, a well-established research field with deep roots in experimental design~\citep{chaloner1995bayesian} and active learning~\citep{settles2009active}, including Bayesian disagreement-based acquisition for deep models~\citep{gal2017deep, kirsch2019batchbald}.
While a large body of literature has proposed domain-agnostic approaches~\citep{ash2019deep,holzmuller2023framework}, several recent works have focused on active learning in MDPs, particularly for learning policies~\citep{judah2012active, he2025uncertainty}. Several existing approaches operate offline \citep{hejna2025robot, agia2025cupid} and filter a fixed dataset based on mutual information between states and actions or influence detection in a static setting. A long-running line of research \citep{cui2019uncertainty, karli2025ask} operates instead in an online setting and aims to infer \textit{when} a human labeling effort is required. Closer to our approach, AMF~\citep{bagatella2025active} also focus on an iterative setting but propose an algorithm that queries \textit{entire} expert demonstrations to maximize information gain over multi-task expert trajectories. While this approach has strong regret guarantees, it does not scale well to generalist policies, unlike our proposed method.

\paragraph{Failure Detection for Generative Policies.}
A common proxy for detecting failures of policies trained via imitation learning is to identify out-of-distribution~(OOD) inputs in the policy’s observation space. \cite{xu2025can} fit a continuous normalizing flow to observation embeddings from successful rollouts and score new observations by their likelihood under the learned distribution, while FIPER~\citep{romer2025failure} uses random network distillation~\citep{burda2019exploration} to detect observation embeddings that fall outside the successful-rollout distribution. However, for modern generalist policies, distribution shift in the input observations is not, by itself, a reliable indicator of impending failure. Other approaches query an external VLM to monitor task progress, but the computational overhead limits real-time deployment~\citep{duan2024aha, du2023vision, agia2024unpacking}.
\cite{gu2025safe} show that VLA latent features separate successful from failed rollouts across tasks and train a lightweight supervised classifier on these features. However, their method requires labeled failure data, which can be unsafe and costly to collect on real robots.
\looseness -1

\section{Preliminaries}
\label{sec:preliminaries}

\textbf{Bayesian Uncertainty Quantification.}
Given a dataset~${\D=\{(\bx^{(n)}, \by^{(n)})\}_{n=1}^N}$ and a neural network parameterized by~$\bth \in \Theta$, trained to learn the relationship between~$\by \in \mathcal{Y}$ and~$\bx \in \mathcal{X}$, Bayesian modeling considers the posterior predictive distribution
\begin{align}
    \label{eq:pre_bayes_predictive_dist}
    p(\bx \condon \by, \D) = \int p(\bx \condon \by, \bth) p(\bth \condon \D)\der \bth,
\end{align}
where~$p(\bth \condon \D)$ is the posterior and generally intractable for deep neural networks~\citep{hullermeier2021aleatoric}. Therefore, \eqref{eq:pre_bayes_predictive_dist} has to be approximated by a Monte-Carlo (MC) estimate~${p(\bx \condon \by, \D) \approx \frac{1}{M} \sum_{i=1}^M p(\bx \condon \by, \bth_i)}$, where weight samples from the posterior~${\bth_1,\dots,\bth_M \sim p(\bth \condon \D)}$ can be obtained using approaches such as ensembling~\citep{lakshminarayanan2017simple} or dropout~\citep{gal2016dropout}. 
The total uncertainty in the generation~$\bx$ for an input~$\by$ is given by the entropy of the predictive distribution~\eqref{eq:pre_bayes_predictive_dist}, ${H\big[p(\bx \condon \by, \D)\big] = -\int p(\bx \condon \by, \D) \log{p(\bx \condon \by, \D)}\der \bx}$,
%
%
which includes both irreducible aleatoric uncertainty and epistemic uncertainty, which can be reduced by collecting additional data.
Subtracting the aleatoric component, given by the expected entropy, yields the epistemic uncertainty
\begin{align}
    \label{eq:pre_bayes_mutual_information}
    I(\bx, \bth \condon \by, \D) = 
    H\big[p(\bx \condon \by, \D)\big] \! - \!\E_{\bth}\!\left[H\big[p(\bx \condon \by, \bth)\big]\right] = \E_{\bth}\!\left[\kl(p(\bx \condon \by, \bth)\,||\,p(\bx \condon \by, \D))\right],
\end{align}
which equals the mutual information between~$\bx$ and~$\bth$ and can be viewed as the expected disagreement between 
a model learned from the data and the true posterior predictive~\eqref{eq:pre_bayes_predictive_dist} at a test input~$\by$.

\textbf{Flow Matching.}
Generative modeling aims to learn an unknown distribution~$q$ from a set of samples~$\mathcal{D}\sim q$.
We first consider unconditional generation, where~${\D = \{\bx^{(n)}\}_{n=1}^N}$, ${\bx^{(n)} \in \R^d}$.
FM models a time-dependent velocity field~${\bu_s: [0,1] \times \R^d \rightarrow \R^d}$ generating a probability path~$(p_s)_{0\leq s \leq 1}$ that transforms a source distribution~${p_0}$ (e.g., a Gaussian) to the data distribution~${p_1 = q}$.
This means the flow~${\bph_s:[0,1] \times \R^d \rightarrow \R^d}$ solving the ODE~${\frac{\der}{\der s}\bph_s(\bx) = \bu_s(\bph_s(\bx))}$, ${\bph_0(\bx) = \bx}$ satisfies~${\bx_s = \bph_s(\bx_0) \sim p_s}$ for~$\bx_0 \sim p_0$. As~$\bu_s(\bx)$ is unknown, conditional FM~\citep{lipman2023flow} with~${\D=\{(\bx^{(n)}, \by^{(n)})\}_{n=1}^N}$ constructs~$p_s$ as the marginal of simpler conditional paths~${p_s(\bx \condon \bx_1)}$, where~${\bx_1 \sim q}$.
A common choice is optimal transport~(OT) Gaussian conditional paths~$p_s(\bx \condon \bx_1) = \mathcal{N}(\bx\condon s \bx_1, (1-s)^2 \bm{I})$, for which the conditional velocity field~$\bu_s(\bx \condon \bx_1)$ has an analytical expression.
A neural network~${\bv_s^{\bth}}$ is then trained to regress onto this target vector using the loss
%
\begin{align}
    \label{eq:pre_fm_loss_conditional}
    \mathcal{L}(\bth) = \mathbb{E}_{s \sim \mathrm{Unif}([0,1]),(\bx_1, \by) \sim \mathrm{Unif}(\D), \bx_s \sim p_s(\cdot \condon \bx_1)}\left[\left\|\bv_s^{\bth}(\bx_s, \by) - \bu_s(\bx_s \condon \bx_1)\right\|_2^2\right].
\end{align}
New samples~$\bx \sim p_1(\cdot \condon \by)$ from the learned distribution can be generated, for example, through Euler integration starting from~${\bx_0 \sim p_0}$ with step size~$\delta s = 1/N_\text{s}$ and
$\bx_{s+\delta s} = \bx_s + \bv_s^{\bth}(\bx_s,\by)\delta s$.

\vspace{-1mm}
\textbf{Flow-Based Generalist Policies.}
We consider robot policies~$\pi_{\bth}:\mathcal{O}\rightarrow \Delta(\mathcal{A})$ that map multimodal observations~$\bo_t \in \mathcal{O}$ to distributions over actions~$\bm{a}_t \in \mathcal{A}$, where an observation~$\bo_t=(\bs_t,\bm{I}_t,\bm{l})$ typically contains the proprioceptive state~$\bs_t$, camera images~$\bm{I}_t$ and language instruction~$\bm{l}$.
Flow-based policies are trained through behavior cloning to match the distribution underlying large-scale datasets~${\D=\{(\bo^{(n)},\ba^{(n)})\}_{n=1}^N\sim q}$, using the loss~\eqref{eq:pre_fm_loss_conditional} with~$\by=\bo_t$.
The choice of output differs between model classes: VLAs usually generate action chunks~${\bx=\bA_t=(\ba_{t},\ba_{t+1},\dots,\ba_{t+H-1})}$, whereas WAMs are trained to also predict future observations, i.e., ${\bx=(\bA_t, \bm{O}_t)}$.
\section{Estimating Epistemic Uncertainty of Flow-Matching Models} \label{sec:epistemic_uncertainty_estimation}


Training an FM model yields an estimate~$\bth$ of the model parameters given the data~$\D$.
For a new conditioning input~$\by$, how can we quantify whether the model exhibits high epistemic uncertainty that could be reduced by collecting more data?
A simple estimate of total predictive uncertainty is the conditional entropy $H[p(\bx \mid \by, \bth)]$. 
However, as this quantity conflates epistemic and aleatoric uncertainty, we need to isolate the epistemic component~\eqref{eq:pre_bayes_mutual_information}.
To this end, we denote~$p^{\bth_i}(\bx \condon \by)=p(\bx \condon \by, \bth_i)$ and consider the MC approximation of~\eqref{eq:pre_bayes_mutual_information}, that is,
%
\begin{subequations}
\begin{align}
    \label{eq:meth_mutual_information_approx}
    I(\bx, \bth \condon \by, \D) &\approx \frac{1}{M}\sum_{i=1}^M \kl\left(p^{\bth_i}(\bx \condon \by) \,||\, p(\bx \condon \by, \D)\right) \\
    &\approx \frac{1}{M} \sum_{i=1}^M \kl\Big(p^{\bth_i}(\bx \condon \by) \,||\, \frac{1}{M} \sum_{j=1}^M p^{\bth_j}(\bx \condon \by) \Big) \\
    \label{eq:meth_epistemic_unc_inequality}
    &\leq \frac{1}{M^2} \sum_{i,j=1}^M \kl\left(p^{\bth_i}(\bx \condon \by) \,||\, p^{\bth_j}(\bx \condon \by) \right),
\end{align}
\end{subequations}
where~$\bth_1,\dots,\bth_M\sim p(\bth \condon \D)$, and~\eqref{eq:meth_epistemic_unc_inequality} follows by applying Jensen's inequality to each
term in the outer sum, using the convexity of the KL divergence in its second argument.
Hence, the average pairwise KL divergence between FM models trained on the same data~$\mathcal{D}$, evaluated for a conditioning input~$\by$, represents an approximate upper bound of the epistemic uncertainty.

In principle, the KL divergence between two FM models can be estimated by computing the likelihoods of individual samples, which entails solving an augmented ODE that involves the divergence of the learned velocity field. 
However, computing this divergence is prohibitively expensive in high dimensions.
To make computation tractable, we first derive a relationship between pairwise KL divergence and the marginal velocity fields. This identity motivates a practical disagreement estimator based on learned velocity fields. 
For the proof, we refer to Appendix~\ref{proof:velocity_difference}.


\begin{theorem}
\label{the:kl_velocity_difference}
Consider two distributions~$p^1(\bx \condon \by)$, $p^2(\bx \condon \by)$ induced by velocity
fields~$\bu_s^1(\bx,\by)$, $\bu_s^2(\bx,\by)$ with OT Gaussian conditional probability paths satisfying regularity assumptions (cf. Appendix~\ref{proof:velocity_difference}). Then, the KL divergence between the two distributions is
\begin{align}
\label{eq:kl_velocity_difference}
D_{\text{\normalfont KL}}\big(
p^1(\bx \condon \by)
\,||\,
p^2(\bx \condon \by)
\big)
&=
\int_{0}^1
\frac{s}{1-s}
\mathbb{E}_{\bx_s \sim p_s^1(\bx \condon \by)}
\left[
\left\|
\bu_s^1(\bx_s,\by)
-
\bu_s^2(\bx_s,\by)
\right\|_2^2
\right]
\mathrm{d}s.
\end{align}
\end{theorem}
For trained FM models, we can apply this disagreement functional to the learned velocity fields.
In this way, we can approximate the KL divergence terms in~\eqref{eq:meth_epistemic_unc_inequality} by sampling~$\bx_0\sim p_0$,
performing Euler integration using the learned velocity field~$\bv_s^{\bth_1}$ and calculating the \textit{velocity-field disagreement~(VFD)} at intermediate ODE states~$\bx_0,\bx_{\delta s},\dots$, weighted by~$\kappa_s=\frac{s}{1-s}$, as
\begin{align}
    \label{eq:meth_kl_velocity_difference_approximation}
    D_{\text{\normalfont KL}}\big(p^{\bth_1}(\bx &\condon \by)\,||\,p^{\bth_2}(\bx \condon \by)\big) \approx 
    \frac{1}{N_s}\mathbb{E}_{\bx_0\sim p_0}\Bigg[
    \sum_{\ell=0}^{N_s-1}
    \kappa_{s_\ell} \left\|\bv_{s_\ell}^{\bth_1}(\bx_{s_\ell}, \by) - \bv_{s_\ell}^{\bth_2}(\bx_{s_\ell}, \by)\right\|_2^2\Bigg],
\end{align}
where~$s_\ell = \ell\delta s$.
Finally, by denoting~$\mathcal{V}=(\bv_s^{\bth_1},\dots,\bv_s^{\bth_M})$,  plugging~\eqref{eq:meth_kl_velocity_difference_approximation} into~\eqref{eq:meth_epistemic_unc_inequality} and averaging over non-identical pairs, we can define the tractable \textbf{VFD score for epistemic uncertainty estimation}
\begin{align}
    \label{eq:meth_vfd_score}
    u_\mathrm{e}(\by; \mathcal{V}) = \frac{1}{M(M-1) N_s} \mathbb{E}_{\bx_0\sim p_0} \Bigg[ \sum_{
    i,j=1;i\neq j
    }^M \sum_{\ell=0}^{N_s-1} \kappa_{s_\ell} \left\|\bv_{s_\ell}^{\bth_i}(\bx_{s_\ell}^{(i)}, \by) - \bv_{s_\ell}^{\bth_j}(\bx_{s_\ell}^{(i)}, \by)\right\|_2^2\Bigg],
\end{align}
where~$\bx_{s_\ell + \delta s}^{(i)} = \bx_{s_\ell}^{(i)} + \bv_{s_\ell}^{\bth_i}\big(\bx_{s_\ell}^{(i)}, \by\big) \delta s$, $\bx_0^{(i)} = \bx_0$.
In practice, we approximate the expectation in~\eqref{eq:meth_vfd_score} by forward-integrating a small batch of actions, leveraging GPU parallelization. 

This uncertainty estimate can be used in several ways to make generalist robot policies more reliable and adaptable, with \textit{failure detection} and \textit{active learning} being of particular interest. The former deploys uncertainty in a relatively straightforward way, flagging a failure when uncertainty is high~\citep{xu2025can, romer2025failure}.
Active fine-tuning involves more complex design choices, and we present an effective algorithm that utilizes VFD in the following section.


\section{Uncertainty-Guided Active Multitask Fine-Tuning} \label{sec:method}

We consider the following problem: Given a flow-based generalist policy and a set of tasks, how can we improve overall performance across tasks with minimum data collection effort?
More formally, we consider a set of $K$ tasks~${\mathcal{T}=\{T_k\}_{k=1}^K}$, each of which can be described by an MDP~${\mathcal{M}_k=(\mathcal{O}, \mathcal{A}, P_k, R_k, \rho_k, \gamma)}$, where ${P_k: \mathcal{O} \times \mathcal{A} \to \Delta(\mathcal{O})}$, ${R_k: \mathcal{O} \to \mathbb{R}}$ and ${\rho_k \in \Delta(\mathcal{O})}$ are task-specific dynamics, rewards and initial observation distributions, and $\gamma\in(0,1)$ is a discount factor. Given a policy~${\pi: \mathcal{O} \to \Delta(\mathcal{A})}$, its single-task performance is simply~${J_k(\pi) = \mathbb{E}_{\pi, P_k, \rho_k} \sum_{t=0}^{\infty} \gamma^t R_k(\bo_t)}$, and its multitask performance is~${J(\pi) = \frac{1}{K} \sum_{k=1}^{K} J_k(\pi)}$. Our goal is to fine-tune a generalist policy~$\pi_{\bth}$ (pre-trained on large-scale, diverse data) to maximize its average multitask performance.
For each task~$T_k \in \mathcal{T}$, we assume the availability of a pool of~$L$ candidate initial observations~$\mathcal{O}_k = \{\bo_{kl}\}_{l=1}^{L} \sim \rho_k$ which, in practice, could include scenarios the policy might encounter at the start of a rollout.
Without any environment interaction, a subset of observations~${\hat{\mathcal{O}} \subset \cup_{k=1}^{K} \mathcal{O}_k}$ needs to be selected, and an expert demonstrator~${\pi^* \approx \mathrm{argmax}_\pi J(\pi)}$ will return a demonstration starting from each of the selected observations: ${\tau^*=(\bo_0, \ba_0, \dots)}$ with $\ba_t \sim \pi^*(\bo_t)$ and ${\bo_{t+1} \sim P_k(\bo_t, \ba_t)}$ for each~${\bo_0 \in \mathcal{\hat O}}$.

Active data acquisition requires deciding which demonstrations are \textit{most valuable to collect}.
We address this problem by using VFD to measure epistemic uncertainty for candidate initial configurations.
Higher values indicate the policy is more uncertain about how to act in a given scene, suggesting limited exposure during pre-training and a potential benefit from additional expert demonstrations.

\paragraph{Generalist Policy Ensemble.}
The high cost of pre-training foundation models makes training multiple generalist policies from scratch prohibitively expensive.
Instead, we take a single pre-trained base policy~$\pi_\text{b}$ and fine-tune it~$M$ times on a diverse dataset~$\mathcal{D}_{\text{pre}}$, randomly shuffled each time.
In this way, we obtain a policy ensemble~${\Pi_0 = \big(\pi_0^{(1)},\dots, \pi_0^{(M)}\big)}$ for uncertainty estimation.


\paragraph{Uncertainty-Guided Data Acquisition.}
We train the ensemble over $R$ rounds of active fine-tuning. At each round $r$, we evaluate the current ensemble $\Pi_r$ on all candidate initial observations and compute, for all tasks~$k \in \{1,\dots,K\}$ and candidate initial observations~$\bo_{kl} \in \mathcal{O}_k$, the VFD epistemic uncertainty
$u_{kl}^{(r)} = u_\text{e}(\bo_{kl}; \Pi_r)$, as defined in~\eqref{eq:meth_vfd_score}. 
A simple strategy would then be to always select the most uncertain candidate observation.
However, this approach can overemphasize a small set of outlier scenarios, leading to poor task coverage and small multitask improvement.
Thus, we first aggregate uncertainty at the task level using the mean score
\begin{align}
    \label{eq:aggregate_uncertainties}
    U_k^{(r)} = 
    \frac{1}{L} 
    \big(u_{k1}^{(r)}+ \ldots +u_{kL}^{(r)}\big),
    \qquad k=1,\dots,K.
\end{align} 
It may then be tempting to directly request a batch of demonstrations from the most uncertain task. However, this is both theoretically flawed, as mutual information would need to be updated as soon as a single new demonstration is collected, and practically suboptimal, as an excessive number of demonstrations could be collected for a simple task.
While posterior uncertainty can be computed under strong assumptions about the model class (e.g., analytically, without even observing $\bx$ in the case of Gaussian processes \citep{bagatella2025active}), its estimation with neural networks remains, in general, intractable.
As collecting demonstrations in batches remains practically necessary for computational reasons, we instead introduce diversity by defining a categorical sampling distribution~$\mathrm{Cat}(W^{(r)})$ with a temperature parameter~$\tau \geq 0$ and weights~$W^{(r)}=(w_1^{(r)},\dots,w_K^{(r)})$ over tasks, where~$w_k^{(r)} = \big(U_k^{(r)}\big)^\tau\mathop{/}\sum_{k'=1}^{K}\big(U_{k'}^{(r)}\big)^\tau$.
Larger values of $\tau$ concentrate probability mass on the most uncertain tasks, and $\tau=0$ corresponds to uniform sampling.
We perform $n_\mathrm{e}$ expert queries per round and, for each, first sample a task index $k \sim \mathrm{Cat}(W^{(r)})$, effectively prioritizing tasks with high estimated uncertainty while retaining exploration over all tasks. We then select the most uncertain initial observation within the sampled task,
\begin{align}
\label{eq:meth_most_uncertain_initial_state}
\bo_k^* = \arg\max_{\bo_{kl}\in\mathcal{O}_k} u_{kl}^{(r)},
\end{align}
and request an expert demonstration starting from $\bo_0=\bo_k^*$; $\tau^*=(\bo_0,\ba_0,\dots)$. Repeating this procedure $n_\mathrm{e}$ times yields a batch of newly collected demonstrations $\mathcal{D}_{\mathrm{new}}^{(r)}$.

\paragraph{Iterative Fine-Tuning.}
Fine-tuning a generalist policy solely on new demonstrations can lead to forgetting of previously acquired capabilities~\citep{du2026memory, romer2026clare}. To avoid this issue, we fine-tune the policy ensemble on a mixture of pre-training and queried data, $
\mathcal{D}_{\mathrm{train}}^{(r)}
\sim
\lambda\; \mathrm{Unif}(\mathcal{D}_{\mathrm{pre}})
+
(1-\lambda)\;\mathrm{Unif}(\mathcal{D}_{\mathrm{new}}^{(\le r)})
$,
where $\mathcal{D}_{\mathrm{new}}^{(\le r)}$ contains all demonstrations collected up to round $r$, and $\lambda\in[0,1]$ is the replay ratio.
By determining which tasks currently need data most and then selecting the highest-uncertainty initial observations \textit{within} those tasks, \method balances exploitation of uncertain scenes with exploration across tasks for sample-efficient multitask adaptation.



\section{Experiments}
\label{sec:experiments}
Our experiments in simulation and the real world primarily aim at answering three research questions:
\begin{enumerate}[
    label=\textbf{Q\arabic*:},
    nosep,
    itemsep=2pt
]
    \item How well-calibrated is VFD, i.e., to what extent do its values reflect the epistemic uncertainty and task-completion capabilities of flow-based generalist policies?
    \item Can VFD identify where additional data should be collected for active fine-tuning?
    \item How informative are VFD uncertainty estimates for detecting failures during deployment?
\end{enumerate}


\subsection{Experimental Setup}

\paragraph{Simulation Environments.}
We conduct our experiments on three benchmarks.
\pusht~\citep{chi2023diffusion} simulates a 2D agent pushing a T-shaped object into a target area. \libero~\citep{liu2023libero} simulates a Franka manipulator across various household environments, and we focus on its most challenging suite, \liberoten, which includes 10 long-horizon tasks.
To evaluate failure detection capabilities, we use~\liberoplus~\citep{fei2026libero}, which introduces seven perturbation dimensions, including object layouts and camera viewpoints, yielding a wide range of failure modes.


\begin{wrapfigure}[8]{r}{0.45\textwidth}    
\vspace{-13pt}
\centering

\begin{minipage}{0.32\linewidth}
\centering
\begin{overpic}[
    width=\linewidth,
    trim={1.5cm 9.3cm 4.5cm 7cm},
    clip
]{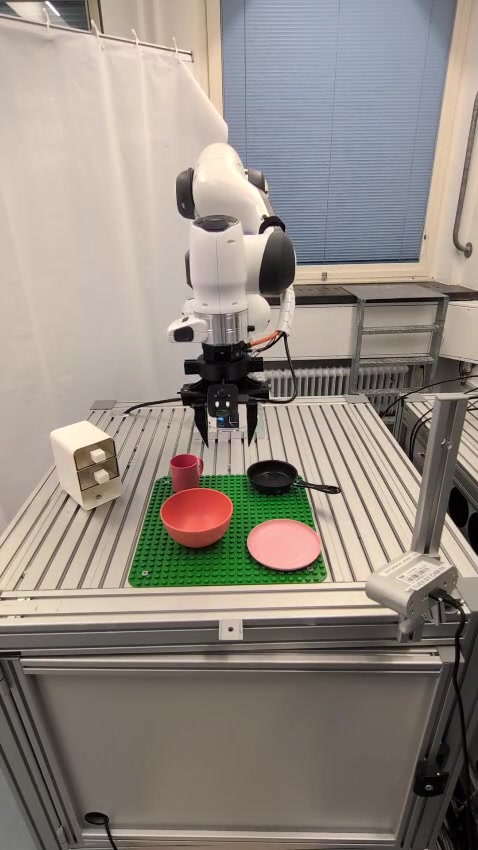}
    \put(0,85){%
\begin{tikzpicture}
    \node[
        fill=black,
        fill opacity=0.5,
        text opacity=1,
        text=white,
        inner xsep=2pt,
        inner ysep=0pt,
        minimum height=1em,
        text height=1.3ex,
        text depth=0.2ex
    ]{\texttt{Place}};
\end{tikzpicture}%
}
\end{overpic}
\end{minipage}
    \hfill
\begin{minipage}{0.32\linewidth}
\centering
\begin{overpic}[
    width=\linewidth,
    trim={1.5cm 9.3cm 4.5cm 7cm},
    clip
]{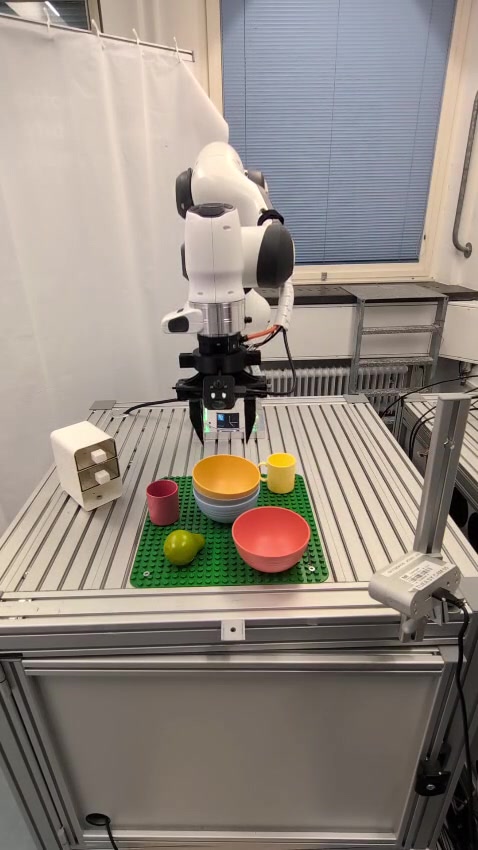}
    \put(0,85){%
\begin{tikzpicture}
    \node[
        fill=black,
        fill opacity=0.5,
        text opacity=1,
        text=white,
        inner xsep=2pt,
        inner ysep=0pt,
        minimum height=1em,
        text height=1.3ex,
        text depth=0.2ex
    ]{\texttt{Stack}};
\end{tikzpicture}%
}
\end{overpic}
\end{minipage}
\hfill
\begin{minipage}{0.32\linewidth}
\centering
\begin{overpic}[
    width=\linewidth,
    trim={1.5cm 9.3cm 4.5cm 7cm},
    clip
]{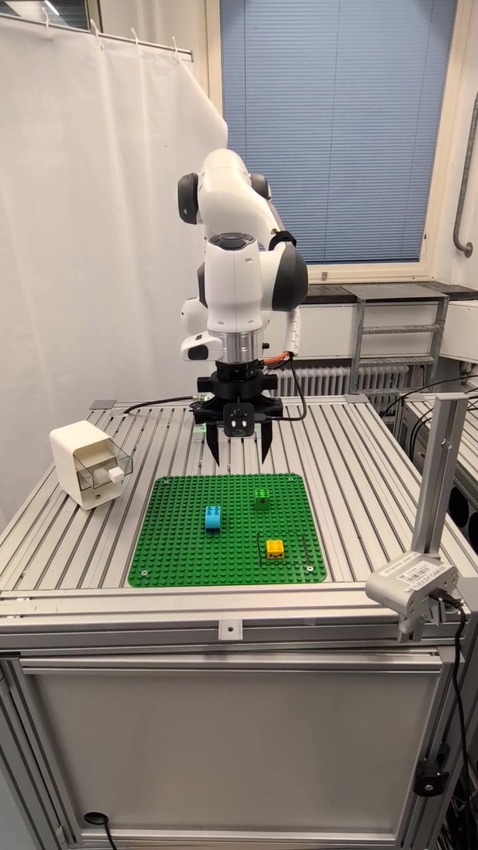}
    \put(0,85){%
\begin{tikzpicture}
    \node[
        fill=black,
        fill opacity=0.5,
        text opacity=1,
        text=white,
        inner xsep=2pt,
        inner ysep=0pt,
        minimum height=1em,
        text height=1.3ex,
        text depth=0.2ex
    ]{\texttt{Drawer}};
\end{tikzpicture}%
}
\end{overpic}
\end{minipage}
    \caption{\textbf{Real-world tasks.}}
    \label{fig:real_world_tasks}
\end{wrapfigure}

\paragraph{Real-World Tasks.}
We design three tasks for an FR3 robot: 1) Placing a bowl on a plate (\texttt{Place}), 2) stacking a bowl on top of three other bowls (\texttt{Stack}), and 3) placing a Lego brick into a drawer and closing it (\texttt{Drawer}). The tasks consist of two, two, and three phases, respectively (e.g., lifting the bowl and placing it down for \texttt{Place}), and we assign partial credit for each successfully completed subtask.

\paragraph{Policies and Implementation Details.}
We use three generalist policies:
SmolVLA~\citep[550M,][]{shukor2025smolva} attaches an action expert to a SmolVLM-2 backbone; X-VLA~\citep[900M,][]{zheng2026xvla} uses soft-prompting for cross-embodiment learning; and FastWAM~\citep[6B,][]{yuan2026fastwam}, built on top of a Wan2.2-5B video DiT backbone, is trained with a joint video modeling and action generation objective.
The three policies are conditioned on a language command, two camera images, and the proprioceptive state.
In~\pusht, we use a simple FM-Policy trained from scratch.
All policies have an ODE step size of ${\delta s = 0.1}$, and we use an ensemble of size~$M=2$.
We repeat all experiments with three seeds (ten for \pusht) and report the mean and standard deviation across seeds.

\paragraph{Iterative Fine-Tuning.}
We investigate uncertainty calibration and active learning using the iterative fine-tuning described in~\Cref{sec:method}.
For~\pusht, we train an ensemble on~\num{50} out of the~\num{206} available demonstrations. For~\libero, $\mathcal{D}_\mathrm{pre}$ contains all demonstrations from the \num{30} tasks in the \texttt{LIBERO-Goal}, \texttt{-Spatial}, and~\texttt{-Object} suites, and we also include three tasks from~\liberoten to instill different levels of initial task competence.
The base ensemble used in the real-world active fine-tuning experiments is trained on $\approx$\num{2500} DROID episodes~\citep{khazatsky2024droid}.
We perform~$R=15$ (\pusht: \num{30}) rounds of iterative active fine-tuning and query~$n_{\text{e}}=5$ (real world: \num{10}) expert demonstrations per round, using a replay ratio of~$\lambda = 0.5$ (\pusht: \num{0.3}).

\subsection{Calibration (Q1): How Well Does VFD Measure Epistemic Uncertainty?} \label{sec:exp_calibration}


\begin{table}[tb!]
    \setlength{\tabcolsep}{4.5pt}
    \centering
    \small
    \caption{\textbf{Calibration analysis.} Negative Spearman's rank correlation ($\uparrow$) between uncertainty estimates and per-task success rates in~\liberoten, averaged across iterative fine-tuning rounds.}
    \begin{tabular}{l l c c c c c c c}
        \toprule
        Policy & Action-L2 & ACE & DECU & GU & Entropy & Perplexity & \textbf{VFD (ours)} \\
        \midrule
        SmolVLA  
            & \num{0.50}$^{\pm \num{0.13}}$ & \num{0.31}$^{\pm \num{0.12}}$ & \num{0.31}$^{\pm \num{0.13}}$ & \num{0.62}$^{\pm \num{0.00}}$ & \num{0.10}$^{\pm \num{0.12}}$ & \num{-0.04}$^{\pm \num{0.09}}$ & $\mathbf{0.71}^{\pm \num{0.03}}$ \\
        X-VLA 
            & \num{0.33}$^{\pm \num{0.24}}$ & \num{0.14}$^{\pm \num{0.05}}$ & \num{0.06}$^{\pm \num{0.04}}$ & \num{0.48}$^{\pm \num{0.06}}$ & N/A & N/A & $\mathbf{0.53}^{\pm \num{0.02}}$ \\
        FastWAM 
            & \num{-0.12}$^{\pm \num{0.06}}$ & \num{-0.29}$^{\pm \num{0.07}}$ & \num{0.03}$^{\pm \num{0.07}}$ & \num{0.27}$^{\pm \num{0.04}}$ & N/A & N/A & $\mathbf{0.44}^{\pm \num{0.17}}$ \\
        \bottomrule
\end{tabular}
\label{tab:calibration_summary}
\end{table}

\paragraph{Setup.}
We measure calibration using Spearman's rank correlation coefficient between the per-task uncertainties~\eqref{eq:aggregate_uncertainties} and success rates, with a value of~$- 1$ indicating that the success rate decreases monotonically with increasing uncertainty.
To avoid bias from data acquisition, we select new demonstrations uniformly at random each round.
We consider six UQ baselines. \textbf{Action-L2} computes the pairwise~$L_2$ distance between actions generated independently from the ensemble members.
\textbf{ACE}~\citep{romer2025failure} computes the conditional entropy in the action chunk distribution. 
\textbf{DECU}~\citep{berry2024shedding} performs pairwise distance estimation at a ``branching'' ODE timestep for denoising diffusion models, and we adapt it for FM.
\textbf{GU}~\citep{jazbec2025generative} generates samples from all ensemble members starting from the same initial noise to approximate the posterior predictive as a Gaussian, and calculates its entropy.
\textbf{Entropy}~\citep{malinin2020uncertainty} and~\textbf{Perplexity}~\citep{fadeeva2024fact} compute average per-token entropy and the exponentiated negative mean of the language tokens' log-probabilities, respectively, and they are only directly applicable to SmolVLA.
\looseness -1

\paragraph{VFD excels at estimating epistemic uncertainty.}
The calibration results, averaged over all iterative fine-tuning rounds, are summarized in~\Cref{tab:calibration_summary}. VFD achieves negative Spearman correlations between~\num {0.05} and~\num{0.17} higher than the best-performing baseline, GU.
The results indicate that VFD uncertainty of a generalist policy on the initial frames is a strong indicator of how difficult the respective task is for that policy, and we examine the implications for active fine-tuning in~\Cref{sec:exp_al}.
\looseness -1

\begin{wrapfigure}{r}{0.5\textwidth}    
    \vspace{-4mm}
    \centering
    \includegraphics[width=\linewidth]{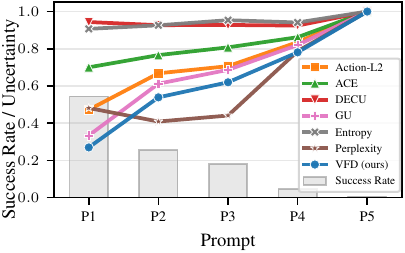}
    \vspace{-6mm}
    \caption{\textbf{Calibration w.r.t. the language prompt.} VFD is well-calibrated when injecting uncertainty only through this modality.}
    \label{fig:language_variation}
    \vspace{-6mm}
\end{wrapfigure}
\textbf{VFD is sensitive to different input modalities.}
As generalist policies often struggle with language conditioning, we investigate whether VFD can capture epistemic uncertainty induced solely by prompt variations.
We generate five increasingly semantically perturbed versions of each language prompt in~\liberoten ($\mathrm{P}1$ to $\mathrm{P}5$), detailed in~\Cref{tab:language_prompts_list}, and evaluate the success rate and uncertainty for each, using the final SmolVLA policy~$\pi_{15}$ after iterative random data acquisition.
The calibration curves are shown in~\Cref{fig:language_variation}, normalized for comparability. VFD yields the strongest monotonic trend, indicating that it can measure epistemic uncertainty induced through different input modalities.
\looseness -1

\subsection{Active Fine-Tuning (Q2): How Well Can VFD Select Which Data to Collect?} \label{sec:exp_al}

\paragraph{Setup.}
We now investigate whether the superior calibration of~VFD transfers to improved data efficiency in active multitask fine-tuning with \method.
As metrics, we measure the area under the learning curve~(AULC) and the average success rate over the last three rounds~(Last-3 SR).
We consider five baselines.
\textbf{Random} selects demonstrations uniformly at random, 
\textbf{Diversity}~\citep{sener2018active} selects the most visually diverse initial observations using a k-center greedy algorithm, and~\textbf{AMF}~\citep{bagatella2025active} performs active multi-task fine-tuning by utilizing posterior uncertainty estimates based on last-layer gradient embeddings.
In addition, we design two baselines based on \method by replacing VFD with GU and Action-L2, which perform second and third in~\Cref{tab:calibration_summary}. 
We optimize hyperparameters by sweeping over temperature values~$\tau\in [0,2.5]$ for \method and over different noise levels for AMF. Further implementation details are provided in Appendix~\ref{sec:experimental_details}.
\looseness -1

\begin{table}[b!]
\setlength{\tabcolsep}{2.25pt}
\centering
\small
\caption{\textbf{Active fine-tuning in simulation.} Area under the learning curve (AULC, $\uparrow$) and average success rate across the last three rounds (Last-3 SR, $\uparrow$) for different data acquisition strategies.}
\label{tab:active_learning_comparison_summary}
\begin{tabular}{lcccccccc}
    \toprule
    Benchmark & \multicolumn{2}{c}{\pusht} & \multicolumn{6}{c}{\liberoten} \\
    \cmidrule(lr){4-9}
    Policy & \multicolumn{2}{c}{FM-Policy} & \multicolumn{2}{c}{SmolVLA} & \multicolumn{2}{c}{X-VLA} & \multicolumn{2}{c}{FastWAM} \\
    \cmidrule(lr){2-3} \cmidrule(lr){4-5} \cmidrule(lr){6-7} \cmidrule(lr){8-9}
    Metric & AULC & Last-3 SR & AULC & Last-3 SR & AULC & Last-3 SR & AULC & Last-3 SR \\
    \midrule
    Random & $38.0^{\pm 3.2}$ & $46.3^{\pm 3.1}$ & $44.6^{\pm 2.2}$ & $52.6^{\pm 3.3}$ & $37.7^{\pm 1.2}$ & $45.1^{\pm 0.3}$ & $54.2^{\pm 2.5}$ & $64.8^{\pm 6.6}$ \\
    Diversity & $40.5^{\pm 1.9}$ & $51.2^{\pm 3.2}$ & $45.5^{\pm 1.8}$ & $53.1^{\pm 2.0}$ & $40.2^{\pm 1.9}$ & $49.0^{\pm 3.3}$ & $53.9^{\pm 0.8}$ & $63.9^{\pm 0.3}$ \\
    AMF & N/A & N/A & $45.6^{\pm 3.3}$ & $59.4^{\pm 3.3}$ & $39.5^{\pm 1.9}$ & $51.0^{\pm 1.3}$ & $55.6^{\pm 1.4}$ & $69.4^{\pm 1.3}$ \\
    \method w/ GU & $44.7^{\pm 2.9}$ & $54.3^{\pm 3.3}$ & $48.6^{\pm 5.0}$ & $61.0^{\pm 2.9}$ & $40.7^{\pm 3.5}$ & $49.2^{\pm 1.8}$ & $59.0^{\pm 4.4}$ & $70.6^{\pm 6.1}$ \\
    \method w/ Action-L2 & $39.8^{\pm 3.4}$ & $45.0^{\pm 2.6}$ & $47.2^{\pm 5.1}$ & $55.5^{\pm 6.8}$ & $39.2^{\pm 0.9}$ & $50.9^{\pm 2.2}$ & $57.0^{\pm 3.3}$ & $71.6^{\pm 5.2}$ \\
    \textbf{\method w/ VFD} & $\mathbf{46.6}^{\pm 2.5}$ & $\mathbf{59.3}^{\pm 3.1}$ & $\mathbf{52.3}^{\pm 2.2}$ & $\mathbf{66.1}^{\pm 2.8}$ & $\mathbf{42.5}^{\pm 1.0}$ & $\mathbf{52.9}^{\pm 2.9}$ & $\mathbf{61.1}^{\pm 3.5}$ & $\mathbf{74.8}^{\pm 3.3}$ \\
    \bottomrule
\end{tabular}
\end{table}

\begin{figure}[t]
    \centering
    \begin{minipage}{0.47\linewidth}
        \centering
        \includegraphics[height=4.4cm]{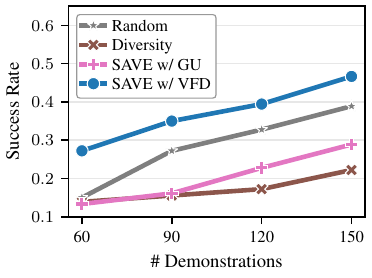}
        \caption{\textbf{Active fine-tuning in the real world.} For fixed demonstration budgets, SAVE w/ VFD improves multitask performance the most.}
        \label{fig:real_world_success_rates}
    \end{minipage}
    \hfill
    \begin{minipage}{0.49\linewidth}
        \centering
        \includegraphics[height=4.4cm]{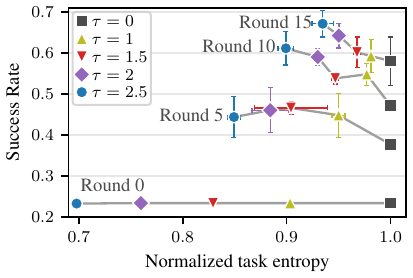}
        \caption{\textbf{Exploration vs. exploitation.} \method w/ VFD prioritizes tasks with higher uncertainty, trading off diversity for multitask improvement.}
        \label{fig:exploration_vs_success_small}
    \end{minipage}
\end{figure}

\paragraph{\method with VFD uncertainty guidance learns new tasks faster.}
\Cref{tab:active_learning_comparison_summary} summarizes the simulation results.
\method w/ VFD consistently yields the highest AULC and Last-3 SR across all policies and environments.
It achieves a~$+1.9$ to~$+5.1$\,pp gain in Last-3 SR over the strongest baseline in each setting and a relative increase of~\qty{3.6}{} to~\qty{7.6}{\percent} in AULC.
Among the uncertainty scores evaluated for~\method, the overall relative ordering is VFD$>$GU$>$Action-L2. This mirrors the calibration results in~\Cref{tab:calibration_summary}, highlighting the importance of well-calibrated uncertainty estimates for guiding active fine-tuning.
In the real world, we consider the strongest uncertainty-based baseline from simulation, SAVE w/ GU, as well as Random and Diversity.
As shown in~\Cref{fig:real_world_success_rates}, SAVE w/ VFD consistently achieves success rates that are only (if at all) reached by the other methods after at least 30 additional demonstrations. At~\num{150} demonstrations, SAVE w/ VFD yields an average task success of~\qty{47}{\percent}, between~$+8$ and~$+25$\,pp higher than the baselines.
The higher multitask improvements show that VFD can allocate fine-tuning budget for generalist policies more effectively by quantifying epistemic uncertainty.
\looseness -1

\paragraph{Uncertainty guidance benefits task and initial state selection.}
We isolate the effects of the two uncertainty-guidance mechanisms of \method on \liberoten with SmolVLA. 
Compared to selecting tasks uniformly at random ($\tau=0$), SAVE w/ VFD increases Last-3 SR by~\qty{11.1}{}\,pp, and by ~\qty{3.4}{}\,pp compared to selecting random starting configurations for the sampled tasks.
\Cref{fig:exploration_vs_success_small} visualizes one reason for these performance gains.
While pre-trained generalist policies naturally exhibit varying levels of uncertainty across tasks, \method w/ VFD allocates demonstration budget to the tasks that require it most, trading off diversity in selected tasks for faster multitask improvement.

\subsection{Failure Detection (Q3): Can VFD be Used for Runtime Monitoring?}

\paragraph{Setup.}
We investigate whether high VFD uncertainty \textit{during deployment} indicates imminent task failure.
For the real-world tasks, we create OOD scenarios by adding novel background objects (e.g., pieces of fruit or a pan).
We roll out the policy and calculate VFD scores at each action-generation timestep.
Following~\cite{agia2024unpacking}, we calibrate thresholds from \num{20} (real world: \num{10}) successful rollouts per task, i.e., without failure data, using conformal prediction~\citep{angelopoulos2023conformal}.
We report balanced accuracy and time-step-wise accuracy~(TWA), which rewards correctly detected failures more when an alarm is raised earlier. 
As baselines, we consider detecting OOD observations by applying random network distillation~\citep[{\bf RND-OE},][]{romer2025failure}
or training normalizing flows~\citep[\textbf{logpZO},][]{xu2025can} on the policies' training data, measuring temporal consistency in the action distributions~\citep[\textbf{STAC},][]{agia2024unpacking}, and computing the conditional action-chunk entropy~\citep[\textbf{ACE},][]{romer2025failure}. 
More details are provided in~Appendix~\ref{app:failure_detection_details}.
\looseness -1

\paragraph{VFD detects failures more accurately.}
As reported in~\Cref{tab:failure_prediction_summary},
VFD achieves a balanced accuracy of~\qty{72}{\percent} across all environments, which is \qty{8}{}\,pp higher than the strongest baseline, logpZO. 
While the observation-based OOD detectors logpZO and RND-OE are trained on the same data as the policy ensemble, VFD is much more robust in distinguishing benign OOD situations where the policy succeeds from actual failures, achieving a TNR of~\qty{81}{\percent} compared to~\qty{58}{\percent} and~\qty{54}{\percent}, respectively.
The high accuracy of VFD on~\liberoplus and in the real world demonstrates its potential for runtime monitoring of capable generalist policies that can handle some, but not all, novel scenarios.

\begin{table}[tb!]
    \centering
    \setlength{\tabcolsep}{3.2pt}
    \small
    \caption{\textbf{Failure detection}. VFD achieves the highest balanced accuracy (Acc.) and the best overall trade-off between accuracy and early detection, as measured by timestep-wise accuracy (TWA).}
    \label{tab:failure_prediction_summary}
    \begin{tabular}{l cc cc cc cc}
        \toprule
        & \multicolumn{2}{c}{\pusht}
        & \multicolumn{2}{c}{\liberoplus}
        & \multicolumn{2}{c}{Real World}
        & \multicolumn{2}{c}{Average} \\
        \cmidrule(lr){2-3}
        \cmidrule(lr){4-5}
        \cmidrule(lr){6-7}
        \cmidrule(lr){8-9}
        Method
        & Acc. $\uparrow$ & TWA $\uparrow$
        & Acc. $\uparrow$ & TWA $\uparrow$
        & Acc. $\uparrow$ & TWA $\uparrow$
        & Acc. $\uparrow$ & TWA $\uparrow$ \\
        \midrule
        STAC
        & $0.55^{\pm 0.00}$ & $0.51^{\pm 0.00}$
        & $0.70^{\pm 0.03}$ & $0.59^{\pm 0.02}$
        & $0.43^{\pm 0.02}$ & $0.42^{\pm 0.02}$
        & $0.56^{\pm 0.02}$ & $0.51^{\pm 0.01}$ \\

        logpZO
        & $0.53^{\pm 0.01}$ & $0.51^{\pm 0.02}$
        & $0.68^{\pm 0.02}$ & $0.62^{\pm 0.01}$
        & $0.70^{\pm 0.03}$ & $\mathbf{0.64}^{\pm 0.03}$
        & $0.64^{\pm 0.02}$ & $0.59^{\pm 0.02}$ \\

        RND-OE
        & $0.50^{\pm 0.02}$ & $0.50^{\pm 0.02}$
        & $0.67^{\pm 0.01}$ & $0.62^{\pm 0.01}$
        & $0.69^{\pm 0.04}$ & $0.61^{\pm 0.02}$
        & $0.62^{\pm 0.00}$ & $0.58^{\pm 0.00}$ \\

        ACE
        & $0.55^{\pm 0.06}$ & $0.51^{\pm 0.03}$
        & $0.75^{\pm 0.02}$ & $0.59^{\pm 0.02}$
        & $0.54^{\pm 0.01}$ & $0.50^{\pm 0.00}$
        & $0.61^{\pm 0.02}$ & $0.54^{\pm 0.02}$ \\

        \textbf{VFD}
        & $\mathbf{0.64}^{\pm 0.05}$ & $\mathbf{0.56}^{\pm 0.05}$
        & $\mathbf{0.79}^{\pm 0.00}$ & $\mathbf{0.65}^{\pm 0.00}$
        & $\mathbf{0.73}^{\pm 0.01}$ & $0.58^{\pm 0.00}$
        & $\mathbf{0.72}^{\pm 0.02}$ & $\mathbf{0.60}^{\pm 0.02}$ \\

        \bottomrule
    \end{tabular}
\end{table}

\begin{wraptable}[11]{r}{0.59\textwidth}
    \vspace{-22pt}
    \centering
    \setlength{\tabcolsep}{3pt}
    \small
    \caption{\textbf{Computational cost.} Speeding up the computation of VFD only slightly reduces failure detection accuracy.}
    \label{tab:computational_cost}
    \begin{tabular}{l l c c c}
        \toprule
        \begin{tabular}[t]{@{}l@{}}Posterior \\ approximation\end{tabular} &
        \begin{tabular}[t]{@{}l@{}}\# Flow times, \\ batch size \end{tabular} &
        \begin{tabular}[t]{@{}c@{}}Balanced \\ accuracy\end{tabular} &
        \begin{tabular}[t]{@{}c@{}}Inference \\ time\end{tabular} &
        \begin{tabular}[t]{@{}c@{}}vs. base \\ policy \end{tabular} \\
        \midrule
        Ensemble ($4$) & $4$, $8$ & $0.78^{\pm 0.00}$ & $+$\qty{251}{\milli \second} & $2.48\times$ \\
        Ensemble ($2$) & $4$, $8$ & $0.73^{\pm 0.01}$ & $+$\qty{89}{\milli \second} & $1.53\times$ \\
        Ensemble ($2$) & $1$, $8$ & $0.72^{\pm 0.01}$ & $+$\qty{49}{\milli \second} & $1.29\times$ \\
        Ensemble ($2$) & $1$, $1$ & $0.70^{\pm 0.02}$ & $+$\qty{32}{\milli \second} & $1.19\times$ \\
        Weight perturb. & $1$, $1$ & $0.68^{\pm 0.01}$ & $+$\qty{29}{\milli \second} & $1.17\times$ \\
        \bottomrule
    \end{tabular}    
\end{wraptable}
\vspace{1mm}
\paragraph{Reducing computational cost.} 
\Cref{tab:computational_cost} reports the effect of reducing the computational effort of VFD for the real-world rollouts.
Perturbing the policy's weights with layer-normalized Gaussian noise instead of training a second model reduces balanced accuracy to~\qty{68}{\percent}, which is, however, still~\qty{14}{} to~\qty{25}{}\,pp above the two training-free baselines, STAC and ACE.

\section{Conclusion}
\label{sec:conclusion}
We study the estimation of epistemic uncertainty in flow-based generalist robot policies and derive VFD as a mathematically grounded and computationally tractable method for this purpose. 
Building on this, we propose \method, a framework for multitask active fine-tuning that allocates the demonstration budget to tasks and initial observations that most require it.
VFD is better calibrated than other UQ methods, enabling more data-efficient multitask adaptation with \method, and it can detect policy failures more accurately, particularly in OOD settings.
In summary, our results establish VFD as an effective method to make generalist policies more reliable and adaptable.


\textbf{Limitations.}
We use the uncertainty for the initial observation as a proxy for task difficulty. Our experiments empirically confirm this correlation, but there may be certain scenarios that ``look simple'' but are challenging for the policy (e.g., due to complex contact dynamics).
Moreover, treating uncertainty as the primary signal for determining which data to collect does not account for informational dependencies across tasks. Developing tractable methods to quantify these dependencies and directly maximizing expected information gain for large-scale policies are interesting avenues for future work.




\newpage

\subsubsection*{Acknowledgments}
This work was supported by the German Federal Ministry of Research, Technology and Space (BMFTR) under the Robotics Institute Germany (RIG) funded by BMFTR grant 16ME0997K, the Humboldt Professorship for Robotics and Artificial Intelligence and the Swiss National Science Foundation under NCCR Automation, grant agreement 51NF40 180545.
Ralf Römer is supported by the German Research Foundation (DFG) within the RTG project ConVeY funded by grant GRK 2428. 
Marco Bagatella is supported by the Max Planck ETH Center for Learning Systems.
Daniel Marta is supported by the Knut and Alice Wallenberg Foundation.

\bibliography{ref}
\bibliographystyle{iclr2027_conference}

\appendix

\fi

\ifappendix
\appendix
\ifpaper
\newpage
\addcontentsline{toc}{section}{Appendix} 
\part{Appendix} 
\parttoc 
\newpage
\else
\setcounter{figure}{5}
\setcounter{table}{2}
\setcounter{proposition}{1}
\addcontentsline{toc}{section}{Appendix} 
\part{} 
\parttoc 
\newpage
\fi
\onecolumn
\section{Theoretical Results}
\label{app:proofs}



\subsection{Flow Matching Fundamentals}
Flow matching~(FM)~\citep{lipman2023flow} formulates generative modeling as learning a time-dependent vector field $\bu_s$ that transports samples from a simple base distribution $p_0$ to the data distribution $p_1 = q$.
We consider a noise distribution $p_0(\bx_0) = \mathcal{N}(\bx_0 \condon \bm{0},\bm{I})$ and a data distribution $\bx_1 \sim q$. The OT Gaussian conditional probability path is defined as
\begin{equation}
    \label{eq:app_fm_ot_path}
    p_s(\bx \condon \bx_1) = \mathcal{N}(\bx \condon s\bx_1, (1-s)^2 \bm{I}),
\end{equation}
and the corresponding flow that pushes the noise distribution $p_0$ to $p_s(\bx \condon \bx_1)$ is given by the affine map
\begin{equation}
    \bph_s(\bx_0) = (1-s)\bx_0 + s\bx_1.
\end{equation}
The conditional velocity field $\bu_s(\bx \condon \bx_1)$ defining this flow via the ODE $\frac{\der}{\der s}\bph_s(\bx) = \bu_s(\bph_s(\bx))$, and generating $p_s(\bx \condon \bx_1)$, can be derived as:
\begin{equation}
\label{eq:app_fm_conditional_velocity_field}
\bu_s(\bx \condon \bx_1) = \frac{\der}{\der s}\bph_s(\bx_0) = \bx_1 - \bx_0.
\end{equation}
Marginalizing over the data distribution $q(\bx_1)$ yields the marginal probability path
\begin{equation}
    \label{eq:app_fm_marginal_probabilities}
    p_s(\bx) = \int p_s(\bx \condon \bx_1)q(\bx_1) \der \bx_1,
\end{equation}
and the marginal velocity field is defined as
\begin{equation}
    \label{eq:app_fm_marginal_velocity_field}
    \bu_s(\bx) = \int \bu_s(\bx \condon \bx_1) \frac{p_s(\bx \condon \bx_1)q(\bx_1)}{p_s(\bx)} \der \bx_1.
\end{equation}

Finally, these components must satisfy the continuity equation
\begin{equation}
\label{eq:app_fm_continuity_equation}
    \frac{\partial}{\partial s}p_s(\bx) + \text{div}(p_s(\bx)\bu_s(\bx)) = 0,
\end{equation}
which ensures probability mass conservation.

\subsection{Proof of~\Cref{the:kl_velocity_difference}} \label{app:proof}
To simplify notation, we drop the fixed conditioning input~$\by$ and 
the parameterization of the velocity fields as neural networks in the following derivations, considering two terminal distributions~${p^1(\bx)=p_1^1(\bx)}$, ${p^2(\bx)=p_1^2(\bx)}$ induced by velocity fields~$\bu_s^1(\bx)$, $\bu_s^2(\bx)$ and identical base distributions~${p_0^1(\bx)=p_0^2(\bx)=\gaussian}$.

We impose the following regularity assumptions.
For every \(s\in[0,1)\) and \(m\in\{1,2\}\), we assume that
\(p_s^m:\mathbb{R}^d\to\mathbb{R}\) and
\(\bu_s^m:\mathbb{R}^d\to\mathbb{R}^d\) are continuously differentiable
with respect to~\(\bx\). Furthermore, for every \(s\in[0,1)\), we assume
that there exist constants \(C_s,a_s,R_s>0\) such that, for every
\(m\in\{1,2\}\),
\begin{equation*}
p_s^m(\bx)
\leq
C_s\exp\!\left(-a_s\lVert\bx\rVert_2^2\right)
\qquad
\text{for all } \lVert\bx\rVert_2\geq R_s.
\end{equation*}
We further assume that
\(\bu_s^1(\bx)\), \(\bu_s^2(\bx)\), and
\(\log\!\left(p_s^1(\bx)/p_s^2(\bx)\right)\) have at most polynomial
growth as \(\lVert\bx\rVert_2\to\infty\).

In the proof of~\Cref{the:kl_velocity_difference}, we will use the following lemma.

\begin{lemma} \label{lem:logp_u}
Let $\bu_s(\bx)$ be the marginal velocity field of an FM model trained with Gaussian OT conditional probability paths, and let~$p_s(\bx)$ be the marginal probability path induced by~$\bu_s(\bx)$. Then, for all $s\in [0,1)$,
\begin{align}
    \nabla \log p_s(\bx) = \frac{s \bu_s(\bx) - \bx}{1-s}.
\end{align}
\end{lemma}
\begin{proof}
Applying Tweedie's formula~\citep{efron2011tweedie} to~\eqref{eq:app_fm_ot_path} and~\eqref{eq:app_fm_marginal_probabilities} yields
\begin{align}
    &\mathbb{E}[s\bx_1\condon\bx_s = \bx] = \bx + (1-s)^2 \frac{\nabla p_s(\bx)}{p_s(\bx)} \\ 
    \Leftrightarrow \quad
    & \nabla \log p_s(\bx) = \frac{s \mathbb{E}[\bx_1\condon\bx_s = \bx] - \bx}{(1-s)^2}.
    \label{eq:app_lem1_step1}
\end{align}
Plugging the conditional velocity~\eqref{eq:app_fm_conditional_velocity_field} into the marginal velocity~\eqref{eq:app_fm_marginal_velocity_field} gives
\begin{align}
    \bu_s(\bx) &= \int (\bx_1 - \bx_0) p_s(\bx_1 \condon \bx_s = \bx)\der \bx_1 \\
    &= \mathbb{E}[\bx_1 - \bx_0 \condon \bx_s = \bx] \\
    &= \frac{\mathbb{E}[\bx_1\condon \bx_s = \bx] - \bx}{1-s} \label{eq:app_lem1_step2} \\
    \Leftrightarrow \quad \mathbb{E}[\bx_1\condon \bx_s = \bx] &= \bx + (1-s)\bu_s(\bx), 
    \label{eq:app_lem1_step3}
\end{align}
where~\eqref{eq:app_lem1_step2} follows from~$\bx=(1-s)\bx_0 + s\bx_1 \Leftrightarrow \bx_1 - \bx_0 = \frac{\bx_1 - \bx}{1-s}$ and taking the conditional expectation.
Finally, plugging~\eqref{eq:app_lem1_step3} into~\eqref{eq:app_lem1_step1} yields
\begin{align}
    \nabla \log p_s(\bx) &= \frac{s\bx + s(1-s)\bu_s(\bx) - \bx}{(1-s)^2} \\
    &= \frac{s \bu_s(\bx) - \bx}{1-s},
\end{align}
which concludes the proof.
\end{proof}

\begin{proof}[Proof of~\Cref{the:kl_velocity_difference}] \label{proof:velocity_difference}
Recall that we aim to express the KL divergence between two distributions~$p^1(\bx)$ and~$p^2(\bx)$ parameterized as FM models in terms of their time-dependent velocity fields~$\bu_s^1(\bx)$ and~$\bu_s^2(\bx)$. 
We write
\begin{align}
    \kl(p^1(\bx)\,||\,p^2(\bx)) &= \kl(p^1(\bx)\,||\,p^2(\bx)) - \kl(p_0^1(\bx)\,||\,p_0^2(\bx)) + \kl(p_0^1(\bx)\,||\,p_0^2(\bx)) \\
    &= \int_0^1 \frac{\partial}{\partial s}\kl(p_s^1(\bx)\,||\,p_s^2(\bx))\der s, 
    \label{eq:app_the1_step1}
\end{align}
where the last step follows from the fact that~$p_0^1 = p_0^2$ since both models use the same source distribution. 
Applying the product rule, we can reformulate the term in the integral as
\begin{align}
    \frac{\partial}{\partial s}\kl(p_s^1(\bx)\,||\,p_s^2(\bx)) 
    &= \frac{\partial}{\partial s} \int p_s^1(\bx)\log\frac{p_s^1(\bx)}{p_s^2(\bx)}\der \bx \\
    &= \int \frac{\partial}{\partial s}p_s^1(\bx)\log\frac{p_s^1(\bx)}{p_s^2(\bx)}\der \bx + \int p_s^1(\bx)\frac{\partial}{\partial s}\log\frac{p_s^1(\bx)}{p_s^2(\bx)}\der \bx \\
    &= \int \frac{\partial}{\partial s}p_s^1(\bx)\log\frac{p_s^1(\bx)}{p_s^2(\bx)}\der \bx + \int p_s^1(\bx)\left(\frac{\frac{\partial}{\partial s}p_s^1(\bx)}{p_s^1(\bx)} - \frac{\frac{\partial}{\partial s}p_s^2(\bx)}{p_s^2(\bx)}\right)\der \bx \\
    &= \int \frac{\partial}{\partial s}p_s^1(\bx)\log\frac{p_s^1(\bx)}{p_s^2(\bx)}\der \bx + \underbrace{\int \frac{\partial}{\partial s} p_s^1(\bx)\der \bx}_{=0} - \int \frac{p_s^1(\bx)}{p_s^2(\bx)} \frac{\partial}{\partial s}p_s^2(\bx)\der \bx, 
    \label{eq:app_the1_step2}
\end{align}
where the second term vanishes since $\int \frac{\partial}{\partial s} p_s^1(\bx)\der \bx = \frac{\partial}{\partial s}\int p_s^1(\bx)\der \bx = \frac{\partial}{\partial s} 1 = 0$.
To reformulate the other terms in~\eqref{eq:app_the1_step2}, we proceed as
\begin{align}
    \frac{\partial}{\partial s}\kl(p_s^1(\bx)\,||\,p_s^2(\bx)) 
    &= \int \frac{\partial}{\partial s}p_s^1(\bx)
    \log\frac{p_s^1(\bx)}{p_s^2(\bx)}\der \bx
    - \int \frac{p_s^1(\bx)}{p_s^2(\bx)}
    \frac{\partial}{\partial s}p_s^2(\bx)\der \bx
    \\
    \label{eq:app_the1_step3}
    &= -\int
    \mathrm{div}\!\left(p_s^1(\bx)\bu_s^1(\bx)\right)
    \log\frac{p_s^1(\bx)}{p_s^2(\bx)}\der \bx
    + \int
    \frac{p_s^1(\bx)}{p_s^2(\bx)}
    \mathrm{div}\!\left(p_s^2(\bx)\bu_s^2(\bx)\right)\der \bx
    \\
    \label{eq:app_the1_step4}
    &=
    \int
    p_s^1(\bx)\bu_s^{1,\top}(\bx)
    \nabla\log\frac{p_s^1(\bx)}{p_s^2(\bx)}
    \der \bx
    -
    \int
    p_s^2(\bx)\bu_s^{2,\top}(\bx)
    \nabla\!\left(
    \frac{p_s^1(\bx)}{p_s^2(\bx)}
    \right)
    \der \bx
    \\
    &=
    \int
    p_s^1(\bx)\bu_s^{1,\top}(\bx)
    \nabla\log\frac{p_s^1(\bx)}{p_s^2(\bx)}
    \der \bx
    -
    \int
    p_s^1(\bx)\bu_s^{2,\top}(\bx)
    \nabla\log\frac{p_s^1(\bx)}{p_s^2(\bx)}
    \der \bx
    \\
    &=
    \int
    p_s^1(\bx)
    \left(\bu_s^1(\bx)-\bu_s^2(\bx)\right)^\top
    \nabla\log\frac{p_s^1(\bx)}{p_s^2(\bx)}
    \der \bx
    \\
    \label{eq:app_the1_step5}
    &=
    \int
    p_s^1(\bx)
    \left(\bu_s^1(\bx)-\bu_s^2(\bx)\right)^\top
    \left(
    \nabla\log p_s^1(\bx)
    -
    \nabla\log p_s^2(\bx)
    \right)
    \der \bx.
\end{align}
Here, \eqref{eq:app_the1_step3} follows from the continuity
equation~\eqref{eq:app_fm_continuity_equation}. Equation~\eqref{eq:app_the1_step4}
follows from integration by parts, with the boundary terms vanishing under
the regularity assumptions stated above.

Applying~Lemma~\ref{lem:logp_u} to the last term in the integral~\eqref{eq:app_the1_step5} yields
\begin{align}
    \label{eq:app_the1_step6}
    \nabla \log p_s^1(\bx) - \nabla \log p_s^2(\bx) = \frac{s}{1-s}(\bu_s^1(\bx) - \bu_s^2(\bx)).
\end{align}
Defining~$\kappa_s=\frac{s}{1-s}$ and plugging~\eqref{eq:app_the1_step6} into~\eqref{eq:app_the1_step5} results in
\begin{align}
    \frac{\partial}{\partial s}\kl(p_s^1(\bx)\,||\,p_s^2(\bx)) 
    &= \int p_s^1(\bx)(\bu_s^1(\bx) - \bu_s^2(\bx))^\top (\nabla \log p_s^1(\bx) - \nabla \log p_s^2(\bx)) \der \bx \\
    &= \int p_s^1(\bx) \kappa_s (\bu_s^1(\bx) - \bu_s^2(\bx))^\top (\bu_s^1(\bx) - \bu_s^2(\bx)) \der \bx \\
    \label{eq:app_the1_step7}
    &= \kappa_s \mathbb{E}_{\bx \sim p_s^1(\cdot)} \left[\|\bu_s^1(\bx) - \bu_s^2(\bx)\|_2^2 \right].
\end{align}
Finally, we can insert~\eqref{eq:app_the1_step7} into~\eqref{eq:app_the1_step1} and obtain
\begin{align}
    \label{eq:app_the1_step8}
    \kl(p^1(\bx)\,||\,p^2(\bx)) = \int_0^1 \kappa_s \mathbb{E}_{\bx \sim p_s^1(\cdot)} \left[\|\bu_s^1(\bx) - \bu_s^2(\bx)\|_2^2 \right] \der s,
\end{align}
concluding the proof.
\end{proof}

\begin{remark}
    Although the weighting~$\kappa_s$ in~\eqref{eq:app_the1_step8} diverges as~$s\rightarrow 1$, the VFD uncertainty estimate remains finite.
    In practice, our estimator~\eqref{eq:meth_vfd_score} evaluates the velocity difference on a grid~${s_\ell=\ell \delta s}$ with ${\ell \in \{0,\dots,N_s-1\}}$, so the largest weight is~${\kappa_{1-\delta s}=\frac{1-\delta s}{\delta s}}$, which is finite for~$\delta s > 0$. 
\end{remark}

\section{Experimental Details}
\label{sec:experimental_details}

\subsection{Computational Resources}
Our codebase builds on LeRobot~\citep{cadene2026lerobot}, released under the Apache License, and FIPER~\citep{romer2025failure}, released under the MIT License.
The simulation experiments involving FM-Policy and SmolVLA were run on a compute node with two NVIDIA RTX 4090 GPUs, \num{8} CPU cores, and \num{64} GB of RAM. With this hardware setup, one active learning experiment with \num{15} rounds of uncertainty quantification, episode selection, \num{4000} gradient steps, and \num{30} evaluation rollouts per task takes approximately \num{12.5} hours. 
The simulation experiments involving X-VLA and FastWAM were conducted on NVIDIA H100 and H200 GPUs, respectively.
The models used in our hardware experiments were trained on an NVIDIA H100 GPU, and we ran real-world inference on an NVIDIA RTX 5090 GPU.
The total computing budget for all experiments, including developing the codebase and running preliminary tests, was about \num{10000} GPU hours.

\subsection{Environments}
\paragraph{\pusht.} 
The \pusht benchmark simulates an agent in a 2D environment that must push a T-shaped object into a drawn target configuration. The observation contains the agent's position and an RGB image of the scene, and the actions are position commands.
Despite its low-dimensional action space, \pusht is a relatively challenging environment due to the multimodality of the demonstrations and the high diversity of initial configurations. The latter makes it a suitable benchmark for evaluating the selection of initial states for querying expert demonstrations.

\paragraph{\libero.} 
We use the \libero simulation benchmark \citep{liu2023libero} (MIT License) in our experiments. 
The environments include a Franka robotic manipulator with a parallel-jaw gripper in a kitchen environment.
For each task, up to 50 human expert demonstrations are available.
The observations~$\bo_t$ consist of two RGB images with spatial resolution $256\times256$, one from a fixed workspace view and one from a camera mounted on the robot's end-effector. Additionally, observations include the 8-dimensional proprioceptive state of the arm, comprising the end-effector position (3D), the end-effector orientation parameterized as an axis-angle (3D), and the gripper joint positions (2D). Actions~$\ba_t$ are continuous and 7-dimensional: 3D delta end-effector translation, 3D delta end-effector orientation, and 1D gripper command.

\paragraph{\liberoplus.} 
To evaluate failure detection in OOD settings in simulation, we use~\liberoplus~\citep{fei2026libero}, which perturbs the \libero tasks across seven dimensions: camera viewpoint, robot initial state, language instruction, lighting, background, sensor noise, and object layout, each with several levels of severity.
The observation and action spaces are identical to those in~\libero.

\begin{figure}[!htb]   
\centering
\begin{minipage}{0.3\linewidth}
\centering
\begin{overpic}[
    width=\linewidth,
    trim={1cm 9cm 2.7cm 5cm},
    clip
]{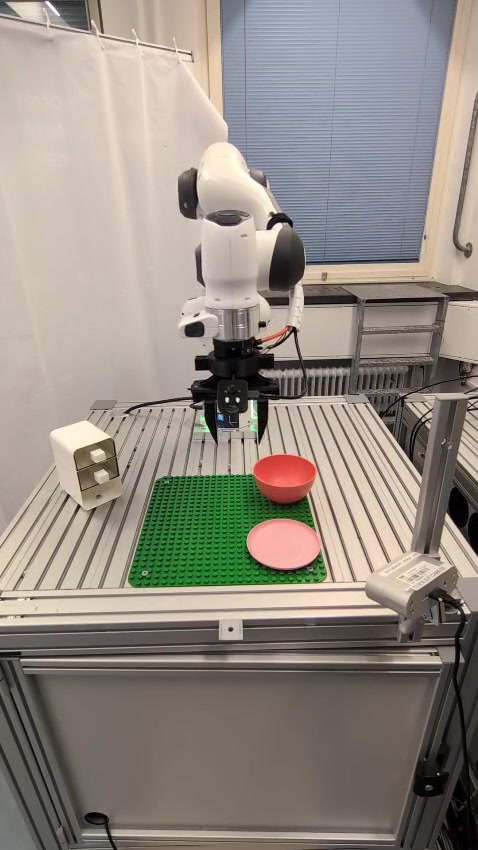}
\end{overpic} \\
\vspace{2mm}
\begin{overpic}[
    width=\linewidth,
    trim={1cm 9cm 2.7cm 5cm},
    clip
]{pics/place_ood_2s.jpg}
\end{overpic} 
\texttt{Place}
\end{minipage}
    \hfill
\begin{minipage}{0.3\linewidth}
\centering
\begin{overpic}[
    width=\linewidth,
    trim={1cm 9cm 2.7cm 5cm},
    clip
]{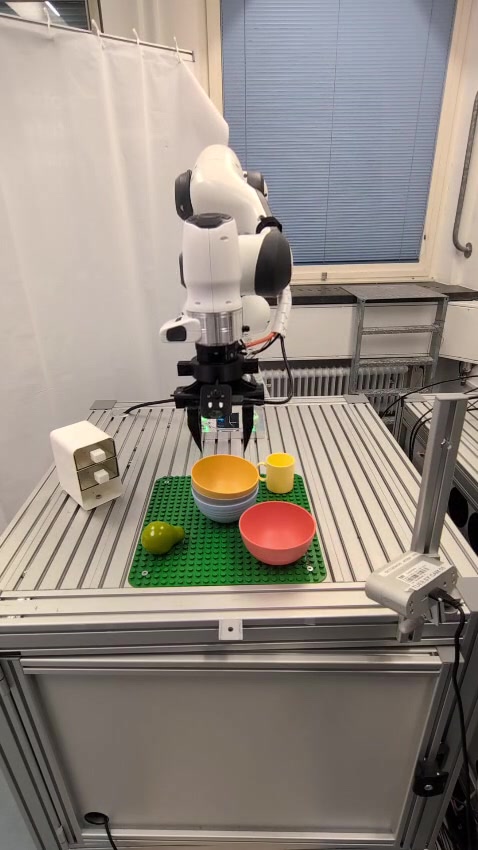}
\end{overpic} \\
\vspace{2mm}
\begin{overpic}[
    width=\linewidth,
    trim={1cm 9cm 2.7cm 5cm},
    clip
]{pics/stack_ood_2s.jpg}
\end{overpic}
\texttt{Stack}
\end{minipage}
\hfill
\begin{minipage}{0.3\linewidth}
\centering
\begin{overpic}[
    width=\linewidth,
    trim={1cm 9cm 2.7cm 5cm},
    clip
]{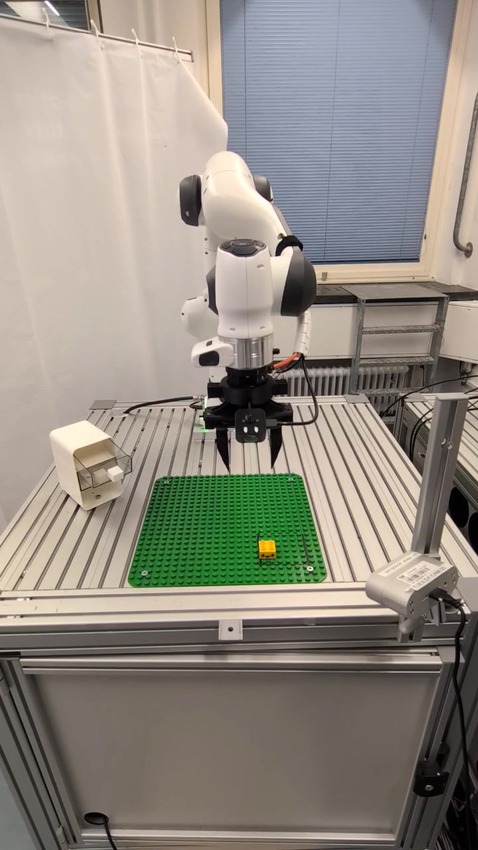}
\end{overpic} \\
\vspace{2mm}
\begin{overpic}[
    width=\linewidth,
    trim={1cm 9cm 2.7cm 5cm},
    clip
]{pics/drawer_ood_2s.jpg}
\end{overpic}
\texttt{Drawer}
\end{minipage}
    \caption{\textbf{Real-world tasks.} Top: nominal settings used to collect demonstrations and evaluate active fine-tuning. Bottom: OOD settings with novel distractor objects in our failure detection experiments.}
    \label{fig:real_world_tasks_big}
\end{figure}

\paragraph{Real-world tasks.} 
We conduct hardware experiments with a Franka FR3 manipulator. The observation contains the language command, RGB images from a wrist-mounted and a third-person camera, the Cartesian end-effector pose and continuous gripper state.
The policies generate desired end-effector poses (position and Euler angles) and gripper state as actions at a control frequency of~\qty{15}{Hz}. The pose targets are tracked by a low-level Cartesian torque controller operating at~\qty {500}{Hz}. 
We define three manipulation tasks involving different objects and motion patterns, each consisting of two to three phases (see~\Cref{fig:real_world_tasks_big}):
\begin{itemize}
    \item \texttt{Place}: Grasp and lift a bowl randomly placed on the table (phase 1) and place it on a plate at a fixed location (phase 2). Language command: \textit{put the bowl on the plate}.
    \item \texttt{Stack}: Grasp and lift a bowl randomly placed on the table (phase 1) and stack it onto three other bowls located at the center of the table (phase 2). Language command: \textit{stack the orange bowl on the other bowls}.
    \item \texttt{Drawer}: Grasp and lift a Lego block pressed into a Lego plate (phase 1), place it into an opened drawer compartment (phase 2), and close the drawer (phase 3). Language command: \textit{put the lego into the drawer}.
\end{itemize}

\subsection{Calibration and Active Fine-Tuning}

\subsubsection{Ensemble Training} \label{app:ensemble_training}
In each experiment, we take a single base checkpoint and fine-tune it into an ensemble using a dataset~$\mathcal{D}_\mathrm{pre}$.
For~\pusht, $\mathcal{D}_\mathrm{pre}$ contains~\num{50} episodes drawn uniformly at random from the available~\num{206} demonstrations.
In our experiments on~\liberoten, $\mathcal{D}_\mathrm{pre}$ contains all episodes from the \texttt{LIBERO-Goal}, \texttt{-Spatial}, and~\texttt{-Object} suites, each covering~\num{10} tasks.
Training on these~\num{30} tasks alone yielded zero initial success rate across the~\liberoten tasks we consider for active fine-tuning. 
Since we aim to test algorithms' capabilities to allocate demonstration budget to the tasks that need it most, we instill varying levels of initial task competence by also including tasks~$(1,2,3)$ from~\liberoten into ~$\mathcal{D}_\mathrm{pre}$.

Additional training details are provided in~\Cref{tab:al_details}.

\subsubsection{Uncertainty Quantification Methods} \label{app:uq_details}
The computation of VFD-based epistemic uncertainty estimation is summarized in~\Cref{alg:uncertainty estimation} and illustrated in~\Cref{fig:vfd_visualization}.
For all policies, including FastWAM, we compute VFD only over the generated action chunks.

In the following, we describe the implementation of the baselines considered in~\Cref{sec:exp_calibration}.

\begin{algorithm}[!htb]
\caption{Efficient Epistemic Uncertainty Estimation for Flow-Matching Models using VFD.} 
\label{alg:uncertainty estimation}
\begin{algorithmic}[1]
    \State {\bfseries Input:} Flow-matching model ensemble~$p^{\bth_1}(\bx\condon \by)$, $p^{\bth_2}(\bx\condon \by)$ with vector fields~${\mathcal{V}=(\bm{v}_s^{\bth_1},\bm{v}_s^{\bth_2})}$ trained via OT paths, conditioning input~$\by$, ODE integration step size~$\delta s=1/N_s$.
    \State \textbf{Output:} Estimate of the epistemic uncertainty~$u_\mathrm{e}(\by; \mathcal{V})$ at the conditioning input~$\by$.
    \State {\bfseries Hyperparameters:} 
    Batch size~$B$.
    \State Sample~$\bx_0^{(1,1:B)} \sim \gaussian$, $\bx_0^{(2,1:B)} \sim \gaussian$. 
    \State Initialize~$u_\mathrm{e}\leftarrow 0$.
    \For{$\ell=0$ {\bfseries to} $N_s-1$} \Comment{ODE integration.}
        \State Set $s = \ell\delta s$.
        \For{{\bfseries each} $i \in \{1,2\}$}
            \State Compute the expected VFDs~${d_{\ell,i}^2 = \frac{1}{B}\sum_{b=1}^B\big\|\bm{v}_s^{\bth_1}\big(\bx_s^{(i,b)},\by\big)-\bm{v}_s^{\bth_2}\big(\bx_s^{(i,b)},\by\big)\big\|_2^2}$.
            \State Update uncertainty estimate $u_\mathrm{e}\leftarrow u_\mathrm{e} + \frac{1}{2 N_s}\frac{s}{1-s}d_{\ell,i}^2$.
            \State Forward integrate~$\bx_{s+\delta s}^{(i,1:B)} = \bx_s^{(i,1:B)} + \bm{v}_s^{\bth_i}\big(\bx_s^{(i,1:B)},\by\big)\delta s$.
        \EndFor
    \EndFor
\State \Return $u_\mathrm{e}(\by; \mathcal{V}) = u_\mathrm{e}$.
\end{algorithmic}
\end{algorithm}

\begin{figure}[!htb]
    \centering
    \input{pics/vfd_illustration_tikz}
    \caption{\textbf{VFD visualization.} We illustrate the computation of VFD for two conditioning inputs with different levels of epistemic uncertainty: low (left) and high (right). Starting from the same initial samples~$\ba_0$, we integrate the ODE of the first ensemble member (blue) and evaluate the velocity field of the second member at the same intermediate states~$\ba_s$ (orange). VFD accumulates the weighted differences between the two velocities (green).}
    \label{fig:vfd_visualization}
\end{figure}

\paragraph{Action-L2 implementation.}
Action-L2 estimates epistemic uncertainty from the disagreement in terminal action distributions. Given an observation $\bo$, we draw independent Gaussian noise samples from $p_0$ and generate $C$ terminal action chunks from the policy and from each ensemble member. Let
\[
    \{\bA^{\mathrm{s}}_i(\bo)\}_{i=1}^C
\]
denote the action chunks generated by the policy, and let
\[
    \{\bA^{(m)}_j(\bo)\}_{j=1}^C
\]
denote the action chunks generated by ensemble member $m \in \{1,\dots,M\}$.

The discrepancy between the policy and ensemble member $m$ is computed as the mean pairwise Euclidean distance between their sampled terminal action chunks:
\[
    d_m(\bo)
    =
    \frac{1}{C^2}
    \sum_{i=1}^C
    \sum_{j=1}^C
    \left\|
        \bA^{\mathrm{s}}_i(\bo)
        -
        \bA^{(m)}_j(\bo)
    \right\|_2 .
\]
The Action-L2 epistemic uncertainty estimate is then the average discrepancy across ensemble members:
\[
    u_{\mathrm{Action\text{-}L2}}(\bo)
    =
    \frac{1}{M}
    \sum_{m=1}^M
    d_m(\bo) .
\]

\paragraph{ACE implementation.}
ACE~\citep{romer2025failure} estimates predictive uncertainty from the dispersion of action samples generated by a single policy. Given an observation $\bo$, we draw $C$ independent Gaussian noise samples and generate $C$ action chunks via FM ODE integration. For each sampled chunk $c \in \{1,\dots,C\}$, we extract the first three (\pusht: two) action dimensions at each timestep~$t$, denoted by~$\{\bm{p}^{(c)}_t\}_{c=1}^C$.
At each timestep $t \in \{1,\dots,H\}$ in the prediction horizon $H$, the sampled positions
$\{\bm{p}^{(c)}_t\}_{c=1}^C$ are discretized into a 3D grid. 
The grid cell size along each coordinate axis is set to $0.03$ times the empirical range of the samples along that axis. 
In our experiments, this performed better than using fixed cell sizes computed offline from calibration data, as proposed by the authors.
Let $\mathcal{B}_t$ denote a set of bins over the occupied grid at time $t$ and let $P_t(b)$ be the empirical probability of a sample falling into bin $b\in \mathcal{B}_t$ at time $t$.
The ACE uncertainty estimate is the Shannon entropy of these bin counts, averaged over the prediction horizon:
\[
    u_{\mathrm{ACE}}(\bo)
    =
    -
    \frac{1}{H}
    \sum_{t=1}^H
    \sum_{b \in \mathcal{B}_t}
    P_t(b)
    \log_2 P_t(b).
\]

\paragraph{DECU implementation.}
DECU~\citep{berry2024shedding} estimates epistemic uncertainty in diffusion models from the disagreement among ensemble members' denoising means at an intermediate point of the generative process. 
Given the close relationship between diffusion and FM, we adapt DECU to FM by replacing the denoising mean with the learned velocity field.
Given an observation $\bo$, we sample $C$ independent Gaussian noise vectors $\bx_0^{(c)} \sim p_0$ for $c \in \{1, \dots, C\}$. Each ensemble member integrates the FM ODE from $s=0$ to a branching time $s_b=0.995$ (as in~\citet{berry2024shedding}) to obtain intermediate states $\bx_{s_b}^{(c)}$. At $s_b$, we evaluate the velocity fields of all $M$ ensemble members, $\bv_i^{(c)} = \bv_{s_b}^{\bth_i}(\bx_{s_b}^{(c)}, \bo)$ for $i \in \{1, \dots, M\}$.

Disagreement is measured using the Pairwise-Distance Estimator (PaiDE). For each sampled chunk $c$, define the squared pairwise velocity-field distances 

\[
    D^{(c)}_{ij}
    =
    \left\|
        \bv^{(c)}_i
        -
        \bv^{(c)}_j
    \right\|_2^2,
    \qquad
    i,j \in \{1,\dots,M\}.
\]

The score for sample $c$ is then

$$
    d^{(c)}_{\mathrm{DECU}}(\bo)
    =
    -
    \frac{1}{M}
    \sum_{i=1}^M
    \log
    \left(
        \frac{1}{M}
        \sum_{j=1}^M
        \exp\!\left(
            -D^{(c)}_{ij}
        \right)
    \right).
$$

The DECU epistemic uncertainty estimate for observation $\bo$ is the mean score over the $C$ sampled chunks:
\[
    u_{\mathrm{DECU}}(\bo)
    =
    \frac{1}{C}
    \sum_{c=1}^C
    d^{(c)}_{\mathrm{DECU}}(\bo).
\]

\paragraph{GU implementation.}
As described in the main text, GU~\citep{jazbec2025generative} approximates the posterior predictive action distribution as a Gaussian. Using our deep ensemble, we compute a raw uncertainty score for each initial observation $\bo_{kl}$, where $k$ indexes the task and $l$ indexes the observation within that task. For a fixed initial noise sample, let 
$
    \bA^{(i)}(\bo_{kl}) \in \mathbb{R}^{H \times D_a}
$
denote the action trajectory generated by ensemble member $i \in \{1,\dots,M\}$, and let
$
    a^{(i)}_{kl,t,d}
$
denote its action value at timestep $t \in \{1,\dots,H\}$ and action dimension $d \in \{1,\dots,D_a\}$.

The raw GU uncertainty score is obtained by summing the log-variance of the ensemble predictions across all timesteps and action dimensions:
\[
    u^{\mathrm{raw}}_{kl}
    =
    \sum_{t=1}^{H}
    \sum_{d=1}^{D_a}
    \log
    \left(
        \operatorname{Var}_{i=1}^{M}
        \left[
            a^{(i)}_{kl,t,d}
        \right]
        +
        \epsilon_{\mathrm{var}}
    \right),
\]
where $\epsilon_{\mathrm{var}} = 10^{-8}$ is added for numerical stability.

Because the differential entropy of a continuous distribution can be negative, these raw scores may also be negative. To obtain non-negative uncertainty weights,
we first aggregate the observation-level scores into a raw task-level uncertainty score $U^{\mathrm{raw}}_k$ and then shift the task-level scores by
\[
    U_k
    =
    U^{\mathrm{raw}}_k
    -
    \min_{k' \in \{1,\dots,K\}}
    U^{\mathrm{raw}}_{k'}
    +
    \epsilon_{\mathrm{off}}
    \left(
        \max_{k' \in \{1,\dots,K\}}
        U^{\mathrm{raw}}_{k'}
        -
        \min_{k' \in \{1,\dots,K\}}
        U^{\mathrm{raw}}_{k'}
    \right).
\]
Here, $\epsilon_{\mathrm{off}} > 0$ is a small offset coefficient. This shift removes negative values, and when the task scores are not all identical, the range-dependent offset moves the minimum-uncertainty task away from zero so that it retains non-zero sampling probability after normalization.

\paragraph{Entropy implementation.}
Entropy~\citep{malinin2020uncertainty} utilizes the extent to which the VLM (SmolVLA backbone) successfully predicts the next token at each valid prefix position. Given an observation $\bo$, the policy constructs a prefix sequence of length $S_\text{prefix}$ containing images, robot states and language. The VLM predicts the next-token probability distribution $\psi_m$ at each valid prefix position $m \in \{1, \dots, S_\text{prefix}\}$, with a single forward pass. The uncertainty score is calculated by the Shannon entropy of this distribution, averaged over all valid prefix tokens:

$$
u_{\mathrm{Entropy}} = - \frac{1}{S_\text{prefix}} \sum_{m=1}^{S_\text{prefix}} \sum_{z \in \mathcal{Z}} \psi_m(z) \log \psi_m(z),
$$

where $\mathcal{Z}$ is the VLM vocabulary and $\psi_m(z)$ is the predicted probability of token $z$ at position $m$.

\paragraph{Perplexity implementation.}
Perplexity~\citep{fadeeva2024fact} estimates uncertainty by measuring the likelihood of the language instruction under the VLM backbone. Similar to Entropy, it evaluates the prefix in a single forward pass without action generation. We extract the predicted next-token log-probabilities for the $S_\text{lang}$ language tokens within their respective prefix. Let $z_m$ be the true language token at sequence position $m$. The uncertainty score $u$ is the exponentiated negative mean log-probability of these tokens:

$$
u_{\mathrm{Perp}} = \exp \left( - \frac{1}{S_\text{lang}} \sum_{m=1}^{S_\text{lang}} \log \psi_m(z_m) \right).
$$

\paragraph{Diversity-Based Greedy.}
The diversity-based greedy baseline~\citep{sener2018active} selects initial observations to maximize visual diversity. Let $\mathcal{O}_k$ denote the pool of candidate initial observations, and $\hat{\mathcal{O}}$ denote the set of selected observations. We extract visual features for each $\bo_i \in \mathcal{O}_k$ using a SigLIP~\citep{zhai2023sigmoid} vision encoder.
We rank all candidates using the $k$-Center Greedy algorithm. At each step, we select $\bo^*$ that maximizes the minimum distance $\Delta$ to $\hat{\mathcal{O}}$:

$$
\bo^* = \arg\max_{\bo_{i} \in \mathcal{O}_k \setminus \hat{\mathcal{O}}} \min_{\bo_{j} \in \hat{\mathcal{O}}} \Delta(\bo_{i}, \bo_{j}),
$$

where $\Delta(\bo_{i}, \bo_{j})$ is the Euclidean distance between visual features. We update $\hat{\mathcal{O}} \leftarrow \hat{\mathcal{O}} \cup \{\bo^*\}$ and repeat until all observations are ranked. In each active fine-tuning round, we select the top $n_\mathrm{e}$ observations from this ranking without replacement.

\subsubsection{Calibration Experiments}

We primarily measure calibration by comparing the uncertainty estimate for an initial observation with the probability of success in rolling out the policy from that observation.
More specifically, we calculate the mean uncertainty~\eqref{eq:aggregate_uncertainties} per task and evaluate the per-task success rate by rolling out the first ensemble member~$\pi_r^{(1)}$ multiple times for different initial states of the robot and the objects in the scene.
To assess whether higher uncertainty correlates with a lower success rate, we calculate the Spearman rank correlation coefficient~$\rho$ between the two quantities.
If and only if the success rate decreases strictly monotonically with higher uncertainty, $\rho=-1$. 
As a secondary metric, we also report the Pearson correlation coefficient, which measures the strength of the linear relationship between the two quantities. 

In~\Cref{fig:language_variation}, we investigate whether the uncertainty scores are well-calibrated when a decrease in success rate stems solely from varying the language command. To this end, we use different language prompts for each task, which are listed in~\Cref{tab:language_prompts_list}.

\begin{table}[!htb]
    \centering
    \small 
    \setlength{\tabcolsep}{3pt} 
    \caption{\textbf{Prompt variations in~\Cref{fig:language_variation}.} We list the prompts used to evaluate calibration of the uncertainty estimators with respect to variations in the language input in \liberoten. Changes from the original prompt are marked in bold.}
    \begin{tabularx}{\textwidth}{l XXXXX} 
        \toprule
        Task & P1 (original) & P2 & P3 & P4 & P5 \\
        \midrule
        1 & put both the alphabet soup and the tomato sauce in the basket & put the alphabet soup and the tomato sauce \textbf{into} the basket & \textbf{place} both the alphabet soup and the tomato sauce in the basket & \textbf{move} both the alphabet soup and tomato sauce \textbf{to} the basket & \textbf{transfer the soup cans to} the basket \\ \addlinespace[0.5em]

        2 & put both the cream cheese box and the butter in the basket & \textbf{place} both the cream cheese box and the butter in the basket & put the cream cheese box and the butter \textbf{into} the basket & \textbf{move} the cream cheese and the butter \textbf{into} the basket & \textbf{place the dairy products} in the basket \\ \addlinespace[0.5em]

        3 & turn on the stove and put the moka pot on it & \textbf{switch} on the stove and \textbf{place} the moka pot on it & \textbf{start} the stove and \textbf{set} the moka pot on \textbf{the burner} & \textbf{activate} the stove and put the moka pot on \textbf{top of} it & \textbf{power} on the \textbf{cooktop} and \textbf{position the coffee maker} on it \\ \addlinespace[0.5em]

        4 & put the black bowl in the bottom drawer of the cabinet and close it & put the black bowl \textbf{into} the \textbf{lower} drawer of the cabinet and \textbf{shut} it & \textbf{place} the black bowl in the \textbf{bottom cabinet drawer} and close it & \textbf{store} the black bowl in the \textbf{bottom cabinet drawer} and close \textbf{the drawer} & \textbf{move the dark} bowl to the \textbf{lowest cabinet compartment} and \textbf{shut} it \\ \addlinespace[0.5em]

        5 & put the white mug on the left plate and put the yellow and white mug on the right plate & \textbf{set} the white mug on the left plate and the yellow and white mug on the right plate & \textbf{place} the white mug on the left plate and \textbf{place} the yellow and white mug on the right plate & put the white \textbf{cup} on the left \textbf{dish} and the \textbf{yellow-white cup} on the right \textbf{dish} & \textbf{arrange the mugs: white one} on the left plate, \textbf{yellow and white one} on the right plate \\ \addlinespace[0.5em]

        6 & pick up the book and place it in the back compartment of the caddy & \textbf{grab} the book and place it in the back compartment of the caddy & pick up the book and \textbf{put} it in the \textbf{rear} compartment of the caddy & \textbf{take} the book and place it in the back \textbf{section} of the caddy & \textbf{move} the book to the back compartment of the \textbf{organizer} \\ \addlinespace[0.5em]

        7 & put the white mug on the plate and put the chocolate pudding to the right of the plate & \textbf{place} the white mug on the plate and put the chocolate pudding to the right of the plate & put the white \textbf{cup} on the \textbf{dish} and \textbf{move} the chocolate pudding to the right \textbf{side of the plate} & \textbf{set} the white mug on the plate and \textbf{place} the chocolate pudding to the right of the plate & \textbf{position the} mug on the plate and \textbf{set the} pudding to the right of \textbf{it} \\ \addlinespace[0.5em]

        8 & place both the alphabet soup and the cream cheese box in the basket & \textbf{put} both the alphabet soup and the cream cheese box in the basket & \textbf{put} the alphabet soup and the cream cheese box \textbf{into} the basket & \textbf{move} both the alphabet soup and cream cheese box \textbf{to} the basket & \textbf{transfer the soup and cheese box to} the basket \\ \addlinespace[0.5em]

        9 & put both moka pots on the stove & \textbf{place} both moka pots on the stove & put the \textbf{two} moka pots on the stove & \textbf{set} both \textbf{coffee} pots on the stove & \textbf{move the two} moka pots \textbf{onto the cooktop} \\ \addlinespace[0.5em]

        10 & put the yellow and white mug in the microwave and close it & \textbf{place} the yellow and white mug in the microwave and close it & put the yellow and white mug \textbf{into} the microwave and \textbf{shut} it & \textbf{move} the yellow and white mug \textbf{into} the microwave and close \textbf{the door} & \textbf{place the yellow-white} mug in the microwave and \textbf{shut the door} \\
        \bottomrule
    \end{tabularx}
    \label{tab:language_prompts_list}
\end{table}

\subsection{Active Fine-Tuning Experiments}


We provide implementation details about our simulated and real-world active fine-tuning experiments in~\Cref{tab:al_details} and the full algorithm in~\Cref{alg:active_learning}.
In our experiments, for a fair comparison, all data-acquisition strategies retrieve demonstrations from the same expert dataset.
An initial configuration can only be selected once per fine-tuning round, but in multiple different rounds.

\begin{table}[!htb]
    \centering
    \small
    \caption{Implementation details for the active fine-tuning experiments.}
    \label{tab:al_details}
    \begin{tabular}{l c c c c c}
        \toprule
        Environment & \pusht & \multicolumn{3}{c}{\liberoten} & Real world \\
        \cmidrule(lr){3-5}
        Policy & FM-Policy & SmolVLA & X-VLA & FastWAM & SmolVLA \\
        \midrule

        \textit{Ensemble training} & & & & & \\
        Demonstrations in $\mathcal{D}_\mathrm{pre}$ & \num{50} & \num{1437} & \num{1437} & \num{1437} & \num{2549} \\
        Training steps & \num{30000} & \num{30000} & \num{30000} & \num{30000} & \num{30000} \\
        Batch size & \num{128} & \num{32} & \num{32} & \num{32} & \num{24} \\
        Learning rate & \num{1e-4} & \num{1e-4} & \num{1e-4} & \num{1e-4} & \num{1e-4} \\
        Learning rate schedule & cosine & cosine & cosine & cosine & cosine \\
        Final learning rate & \num{0} & \num{2.5e-6} & \num{2.5e-6} & \num{2.5e-6} & \num{2.5e-6} \\
        Warmup steps & \num{1000} & \num{1000} & \num{1000} & \num{1000} & \num{1000} \\

        \midrule

        \textit{Iterative fine-tuning} & & & & \\
        Demonstrations per round & \num{5} & \num{5} & \num{5} & \num{5} & \num{10} \\
        Rounds~$R$ & \num{30} & \num{15} & \num{15} & \num{15} & \num{15} \\
        Training steps per round & \num{2000} & \num{4000} & \num{4000} & \num{4000} & \num{10000} \\
        Batch size & \num{128} & \num{32} & \num{64} & \num{64} & \num{32} \\
        Learning rate & \num{1e-4} & \num{5e-5} & \num{1e-4} & \num{1e-4} & \num{5e-5} \\
        Learning rate schedule & cosine & cosine & cosine & cosine & cosine \\
        Final learning rate & \num{0} & \num{5e-6} & \num{5e-6} & \num{5e-6} & \num{5e-6} \\
        Warmup steps & \num{100} & \num{200} & \num{200} & \num{200} & \num{500} \\
        Replay ratio $\lambda$ & \num{0.3} & \num{0.5}  & \num{0.5} & \num{0.5} & \num{0.5} \\
        \midrule

        \textit{Evaluation} & & & & \\
        Evaluated rounds & all & all & all & all & \num{6}, \num{9}, \num{12}, \num{15} \\
        Rollouts per round & \num{200} & \num{300} & \num{300} & \num{300} & \num{30} \\
        \bottomrule
    \end{tabular}
\end{table}

\begin{algorithm}[!htb]
    \caption{Uncertainty-Based Active Fine-Tuning of Flow-Based Generalist Policies with \method.} 
    \label{alg:active_learning}
\begin{algorithmic}[1]
    \State {\bfseries Input:} Base policy~$\pi_\mathrm{b}$, pool of initial configurations~$\{\mathcal{O}_k\}_{k=1}^K$, diverse dataset~$\mathcal{D}_{\mathrm{pre}}$.
    \State {\bfseries Output:} Final policy~$\pi_R$.
    \State {\bfseries Hyperparameters:} Rounds~$R$, expert queries per round~$n_\mathrm{e}$, temperature $\tau$, replay ratio $\lambda$.
    \State From~$\pi_\mathrm{b}$, train an ensemble~$\Pi_0$ on~$\mathcal{D}_{\mathrm{pre}}$.
    \State Initialize buffer~$\mathcal{D}_\mathrm{new}^{(\leq0)} \leftarrow \emptyset$.
    \For{$r = 0$ {\bfseries to} $R-1$}
    \Comment{Active learning loop.}
        \For{$k=1$ {\bfseries to} $K$} 
            \For{{\bfseries each} $\bo_{kl}\in\mathcal{O}_k$}
                \State Compute VFD uncertainty~$u_{kl}^{(r)} = u_\mathrm{e}(\bo_{kl};\Pi_r)$ via~\eqref{eq:meth_vfd_score}.
            \EndFor
            \State Compute task uncertainty $U_k^{(r)}$ via~\eqref{eq:aggregate_uncertainties}.
        \EndFor        
        \State Compute task weights~$W^{(r)}=\big(w_1^{(r)},\dots,w_K^{(r)}\big)$ via~$w_k^{(r)} = \frac{\left(U_k^{(r)}\right)^\tau}{\sum_{k'=1}^{K}\left(U_{k'}^{(r)}\right)^\tau}$.
        \For{$\xi=1$ {\bfseries to} $n_{\text{e}}$} \Comment{Data acquisition.}
        \State Sample $k \sim \mathrm{Cat}(W^{(r)})$.
        \State Select the most uncertain initial observation~$\bo^*$ from task~$T_k$ via~\eqref{eq:meth_most_uncertain_initial_state}.
        \State Query the expert to collect a demonstration~$\tau^*$ starting from $\bo^*$ and add~$\tau^*$ to $\mathcal{D}_{\mathrm{new}}^{(\le r)}$.
        \EndFor
        \State Construct training mixture $\mathcal{D}_{\mathrm{train}}^{(r)} \sim \lambda\; \mathrm{Unif}(\mathcal{D}_{\mathrm{pre}})+(1-\lambda)\;\mathrm{Unif}(\mathcal{D}_{\mathrm{new}}^{(\le r)})$.
        
        \State $\Pi_{r+1}\leftarrow \mathrm{Fine\text{-}tune}(\Pi_r, \mathcal{D}_{\mathrm{train}}^{(r)})$, 
        $D_\mathrm{new}^{(\leq r+1)} \leftarrow D_\mathrm{new}^{(\leq r)}$.
    \EndFor
\end{algorithmic}
\end{algorithm}

\paragraph{AMF implementation.}
The central idea of AMF~\citep{bagatella2025active} is to use the demonstrations collected so far to estimate which task's next demonstration would most improve performance across all tasks. It does so by selecting the task that maximizes the cumulative mutual information between a new demonstration for that task and the policies for \textit{all} relevant tasks.

Since this method uses previously collected demonstrations for task selection, we start with a prefill phase, collecting one episode per task (for our \libero experiments, this takes the first two rounds, using the same acquisition budget as the other methods). We encode frames from previous demonstrations with loss-gradient embeddings from the last action-projection layer (\citet{bagatella2025active} ablate this decision and show slight improvement over last-layer features) and construct a normalized kernel using the inner product of these embeddings. Intuitively, frames whose losses suggest similar parameter updates have high kernel-similarity. For a task with more than one previously collected demonstration, we evaluate each demonstration separately. Using the loss-gradient kernel, we compute how the covariance would change if we received information equivalent to that demonstration. This reduces uncertainty at its frames and at other frames whose gradient embeddings are correlated with them. We aggregate the remaining uncertainty by summing frame variances within each demonstration, averaging across demonstrations within each task, and averaging across tasks equally. We then greedily select the task whose demonstration leaves the lowest remaining total uncertainty until the demonstration budget $n_e$ is reached and perform the same training steps as with the other methods.

\paragraph{Sample Efficiency Computation.}
The goal of active fine-tuning is to improve performance with as little data as possible. To quantify this capability, we measure \textit{sample efficiency} as follows. 
Denote the success rate of method $\mathrm{A}$ after round~$r_1$ by ~$\mathrm{SR}_{A,r_1}$. Let~$r_2=\min\{r:\mathrm{SR}_{B,r} \geq \mathrm{SR}_{A,r_1}\}$ be the number of rounds required for method $B$ to reach at least a similar success rate as~$\mathrm{SR}_{A,r_1}$. Then, if~$r_2 < r_1$, method~$B$ requires $\frac{r_1-r_2}{r_1}\times 100\%$ fewer samples to achieve the same performance, i.e., method $B$ is~$(\frac{r_1}{r_2}-1)\times 100\%$ more sample-efficient than method $A$. To obtain robust results, we average across~$r_1 \in \{R-2,R-1,R\}$. 
With this evaluation protocol, we find that with SmolVLA, \method w/ VFD requires~\qty{50}{\percent} fewer samples than Random, \qty{45}{\percent} fewer samples than Diversity, ~\qty{40}{\percent} fewer samples than \method w/ Action-L2, and \qty{22}{\percent} fewer samples than \method w/ GU, corresponding to~\qty{101}{\percent}, \qty{83}{\percent}, \qty{68}{\percent}, and \qty{29}{\percent} higher sample efficiency, respectively.

\subsection{Failure Detection} \label{app:failure_detection_details}

For our experiments, we build on the failure prediction framework FIPER~\citep{romer2025failure}, released under the MIT License.

\subsubsection{Environment Details}

\Cref{tab:fd_setup} provides details about the rollouts in simulation and the real world. 
Unless otherwise stated, VFD uses a two-member ensemble consisting of the policy being rolled out and evaluated and a second, independently trained model.

\begin{table}[!htb]
    \centering
    \small
    \caption{Details about the failure detection experiments.}
    \label{tab:fd_setup}
    \begin{tabular}{l c c c}
        \toprule
        Environment & \pusht & \liberoplus & Real world \\
        \midrule
        Policy & FM-Policy & SmolVLA & SmolVLA \\
        Action chunk length~$H$ & \num{16} & \num{50} & \num{30} \\
        Action execution steps~$h$ & \num{8} & \num{25} & \num{15} \\
        Tasks & $1$ & $10$ & $3$ \\
        Calibration rollouts per task & $20$ & $20$ & $10$ \\
        Test rollouts & \num{200} & \num{1800} & \num{152} \\
        Policy success rate & $0.20$ & $0.49$ & $0.56$ \\
        Timesteps per rollout & $50$--$300$ & $140$--$520$ & $160$--$500$ \\
        \bottomrule
    \end{tabular}%
\end{table}

\paragraph{\pusht.}
We use the ten FM-Policy ensembles trained on~\num{50} of the~\num{206} demonstrations (\Cref{app:ensemble_training}). Since~\pusht involves no distribution shift, the environment tests the capabilities of failure detectors in in-distribution settings.

\paragraph{\liberoplus.}
The policies are trained on all demonstrations of the \texttt{LIBERO-Spatial}, \texttt{-Object}, \texttt{-Goal}, and \liberoten suites.
The perturbations reduce the success rate of our SmolVLA policies from~\num{0.79} on the unperturbed \liberoten tasks to \num{0.34} under camera perturbations; \num{0.37} under sensor noise; \num{0.39} when perturbing the robot initial state; \num{0.55} under background perturbations; \num{0.60} under lighting changes; \num{0.64} with changed object layout; \num{0.06} for language perturbations.
Due to its strong imbalance of successes and failures, we do not evaluate failure detection on the language perturbation family.
For each task and perturbation family, we record~\num{30} test rollouts, one per perturbation variant, which yields~\num{1800} test rollouts.
To calibrate the thresholds, we use~\num{20} successful rollouts per task, which are generated by rolling out the same policy on the unperturbed \liberoten tasks.

\paragraph{Real world.}
We fine-tune SmolVLA together on all three tasks \texttt{Place}, \texttt{Stack}, and \texttt{Drawer}, using~\num{100} demonstrations per task.
For each task, we record approximately~\num{50} test rollouts with novel background objects, yielding~\num{152} rollouts, of which~\num{85} succeed.
We assess robustness with respect to the calibration data by repeating the evaluation for three independent calibration draws, each using~\num{10} successful in-distribution rollouts per task.

\subsubsection{VFD Score for Failure Detection}
\label{app:fd_vfd_score}
Since runtime monitoring requires low latency, we use the failure detection experiments to investigate whether the computational cost of VFD can be reduced without substantially decreasing performance.
To this end, we modify the VFD score~\eqref{eq:meth_vfd_score} in two ways.
First, we only evaluate the velocity fields along the ODE trajectories of the acting policy.
These trajectories are integrated anyway to generate actions, and the policy's own velocities at their intermediate states are computed as part of this integration, so VFD only adds evaluations of the other ensemble members' velocity fields at those states.
Second, we only evaluate the velocity difference at a subset~$\mathcal{S}=\{s_1<\dots<s_{N_\mathrm{d}}\} \subseteq \{0,\delta s, \dots,1-\delta s\}$ of~$N_\mathrm{d}$ flow times, weighted by the spacing between them, instead of at all~$N_s=1/{\delta s}$ integration steps.
Specifically, we compute the VFD score as
\begin{equation}
    \label{eq:vfd_failure_detection}
    u_\mathrm{e}(\by) = \frac{1}{(M-1)\,B} \sum_{m=2}^{M} \sum_{b=1}^{B} \sum_{k=1}^{N_\mathrm{d}} \kappa_{s_k}\, \Delta s_k \left\|\bv_{s_k}^{\bth_1}\big(\bx_{s_k}^{(b)}, \by\big) - \bv_{s_k}^{\bth_m}\big(\bx_{s_k}^{(b)}, \by\big)\right\|_2^2,
\end{equation}
where~$\bx_{s}^{(b)}$ is the state of the~$b$-th action chunk at flow time~$s$ along the policy's own integration, $\Delta s_k = s_{k+1}-s_k$ for~$k<N_\mathrm{d}$, $\Delta s_{N_\mathrm{d}} = 1-s_{N_\mathrm{d}}$, and~$\kappa_s = \frac{s}{1-s}$ as defined in~\Cref{sec:epistemic_uncertainty_estimation}.
We investigate ensemble sizes up to~$M=4$.

\subsubsection{Evaluation Protocol}

\paragraph{Threshold computation via conformal prediction.}
For each task, we set a constant threshold~$\gamma$ to the empirical $q$-quantile of the maximum score within each calibration rollout~\citep{agia2024unpacking}, with~$q=0.9$ for all methods and environments.
A test rollout is flagged as a failure at the first action-generation timestep at which its score exceeds~$\gamma$.
This corresponds to split conformal prediction~\citep{angelopoulos2023conformal} with the rollout-wise maximum score as the nonconformity score. 

\paragraph{Metrics.}
Let $N_\mathrm{fail}$ and~$N_\mathrm{succ}$ denote the number of failed and successful test rollouts of a task. The true positive rate~(TPR) is the fraction of failed rollouts that are flagged, the true negative rate~(TNR) the fraction of successful rollouts that are not flagged, and the balanced accuracy is~$\mathrm{Acc}=\frac{1}{2}(\mathrm{TPR}+\mathrm{TNR})$.
For a correctly flagged failed rollout~$n$, the detection time~$\mathrm{DT}_n = t^*_n / (T_\mathrm{max}-1)$ is the timestep~$t^*_n$ of the first flagged step, normalized by the length~$T_\mathrm{max}$ of the longest rollout of the respective task, and DT denotes its average over all correctly flagged failures.
If no failure is flagged, $\mathrm{DT}=1$.
Timestep-wise accuracy~\citep{romer2025failure} rewards earlier detections and is calculated as~$  \mathrm{TWA} = \frac{1}{2}\left(\frac{1}{N_\mathrm{fail}} \sum_{n\in\mathcal{T}} \big(1-\mathrm{DT}_n\big) + \mathrm{TNR}\right)$, where~$\mathcal{T}$ is the set of correctly flagged failed rollouts. Hence, $\mathrm{TWA}\leq\mathrm{Acc}$, with equality if every correctly detected failure is flagged at the first step.
In Tables~\ref{tab:failure_prediction_summary}, \ref{tab:failure_prediction_pusht}, \ref{tab:failure_prediction_liberoplus} and~\ref{tab:failure_prediction_real_world}, bold and underlined values denote the best and second-best method per environment and metric.
A detector that raises alarms on very few or instantly on all rollouts can attain a low detection time without being useful. Therefore, we rank DT only among methods with~$\mathrm{TPR}\geq 0.4$ and~$\mathrm{TNR}\geq 0.4$, applied to unrounded values (e.g., the real-world TNR of RND-OE ($0.398$) is displayed as~$0.40$ but excluded from the DT ranking). TPR and TNR are not ranked.

\subsubsection{Baselines}
We evaluate all methods on the same rollouts and using the same procedure to calculate conformal prediction thresholds.

\textbf{STAC}~\citep{agia2024unpacking} measures the temporal consistency of the action chunk distribution. At each policy inference step~$0,h,2h,\dots$, it compares the action chunks~$\bA_{t}$ sampled at the current step with those sampled at the previous step~$\bA_{t-h}$ over their temporally overlapping parts of length~$H-h$. The score is the maximum mean discrepancy between the two sample sets using an RBF kernel and computed only over the action dimensions corresponding to positions. STAC uses a batch size of~\num{128} for~\pusht and~\num{32} for \liberoplus and the real-world experiments.

\textbf{logpZO}~\citep{xu2025can} fits an FM model to the policy's observation embeddings and scores an embedding~$\bm{e}_t$ by the squared norm~$\|\bm{z}_t\|_2^2$ of the noise~$\bm{z}_t=\bm{e}_t-\bv_1(\bm{e}_t)$ that one backward Euler step of the learned velocity field maps it to. Embeddings far from the training distribution map to noise of large norm. Following the authors, the velocity field is parameterized as a 1D convolutional U-Net with channel dimensions~$(256, 512, 1024)$, kernel size~$5$, $8$ groups, and a $128$-dimensional time embedding, trained for~\num{200} epochs with learning rate~$10^{-4}$ and batch size~$128$.

\textbf{RND-OE}~\citep{romer2025failure} applies random network distillation to the observation embeddings. A predictor network is trained to match the output of a fixed, randomly initialized target network, and the failure detection score is defined as the prediction error between the two networks' outputs. Following the authors, we train the predictor with an $\ell_2$ loss using AdamW with learning rate~$10^{-4}$, cosine decay to~$10^{-6}$, weight decay~$10^{-5}$, and batch size~$256$ for up to~\num{250} epochs, with early stopping on a held-out~\qty{10}{\percent} validation split.

\textbf{ACE}~\citep{romer2025failure} is computed from the same position-only action samples as STAC and described in~\Cref{app:uq_details}.

We do not consider SAFE~\citep{gu2025safe} as a baseline because it requires failure data for training. We deliberately calibrate the thresholds only on successful rollouts, which are easier to obtain.

For a fair comparison with VFD, we train logpZO and RND-OE on the same data used to train the policy ensemble. 
In \pusht, we use the~\num{50} demonstrations of the ensemble, and the embedding is the \num{66}-dimensional observation conditioning vector of the FM-Policy. 
For SmolVLA, the embedding consists of the flattened keys and values from the KV cache of the middle VLM layer over the prefix tokens. 
In the real world, we fit one detector per task on its~\num{100} demonstrations. 
In \liberoplus, we fit one detector per policy on its training demonstrations, which are subsampled with a stride of~\num{10} for computational reasons. Since the number of instruction tokens varies across tasks, we zero-pad the embeddings after the instruction block to the length of the longest instruction, such that a single detector can cover all tasks.

\subsubsection{Computational Cost}
\label{app:fd_computational_cost}
\Cref{tab:computational_cost} reports results for the SmolVLA policy used in our real-world experiments. The second row corresponds to the default VFD configuration used in~\Cref{tab:failure_prediction_summary}.
We measure the wall-clock inference time per action-generation step on a single NVIDIA RTX 5090 GPU, averaged over~\num{60} repetitions. 
The base policy, i.e., one VLM prefix pass that computes the KV cache followed by~\num{10} Euler integration steps of the action expert, takes~\qty{170}{\milli\second} to generate a single action chunk.
VFD integrates a batch of~$B$ action chunks and, for each of the~$M-1$ other ensemble members, requires an additional prefix pass and evaluating that member's velocity field at the~$N_\mathrm{d}$ ODE states of the base policy for the flow times in~$\mathcal{S}$.
On this GPU and at~$B=8$, a prefix pass takes approximately~\qty{25}{\milli\second} and a velocity-field evaluation approximately~\qty{14}{\milli\second}.
The inference time column reports the added time per action-generation step, and the relative column reports the total inference time relative to the base policy. 
We find that using a larger ensemble increases balanced accuracy by~\qty{5}{}\,pp, albeit at the cost of~\qty{162}{\milli\second} longer inference.

\paragraph{Weight perturbation as an alternative to ensembling.}
To assess whether VFD can avoid training an ensemble altogether, we replace the second ensemble member with a perturbed version of the policy itself. 
To this end, we add independent zero-mean Gaussian noise with standard deviation~$\sigma\,\mathrm{std}(\bm{W})$ to each weight tensor~$\bm{W}$, i.e., scaled to the spread of the tensor's own values. This corresponds to a single sample from a Gaussian posterior centered at the trained weights with a heuristic, tensor-wise diagonal covariance.
We draw the perturbation once with a fixed random seed and keep it fixed for all rollouts and steps, including the calibration rollouts, so that the thresholds are calibrated for the same perturbed model that is evaluated on the test rollouts.
We perturb the VLM backbone (vision encoder and language model, \num{460.2}M parameters) and the action expert, including its action and time projections (\num{97.4}M parameters). The language-modeling head of the VLM is not used for action generation and remains unchanged.
We evaluate the perturbed model at the recorded ODE states of the policy, at a single flow time~$s=0.8$ and with~$B=1$.
We select~$\sigma$ from combinations of separate noise scales $\sigma_\mathrm{expert},\sigma_\mathrm{VLM}\in\{0, 0.001, 0.005, 0.01, 0.05, 0.1, 0.2, 0.5\}$ for the two modules.
Perturbing only the action expert, which would allow the perturbed model to reuse the policy's KV cache and reduce the relative inference time to~$1.08\times$, yields a balanced accuracy of at most~\num{0.62}.
Balanced accuracy plateaus at~\numrange{0.66}{0.68} for~$\sigma_\mathrm{VLM}\in[0.1,0.2]$ and $\sigma_\mathrm{expert}\in[0.05,0.2]$, and drops below~\num{0.46} for~$\sigma_\mathrm{VLM}=0.5$.
Because the VLM is perturbed, the perturbed model requires its own prefix pass, such that the inference time matches that of using an ensemble.
In summary, this experiment indicates that VFD can yield competitive uncertainty estimates even without training a deep ensemble.
\section{Additional Results} \label{app:additional_results}

\subsection{Toy Example}

We illustrate the capabilities of VFD for quantifying epistemic uncertainty in FM models with a simple example.

\begin{example}
Consider a conditional generative modeling problem with~${\bx=(x_1,x_2)}$, ${\by=(y_1,y_2) \in \R^2}$. 
The ground-truth conditional distribution is a two-component Gaussian mixture with input-dependent means; $q(\bx\condon\by) = 0.5\mathcal{N}(\bx\condon\bm{\mu}_1(\by), \bm{I}) + 0.5\mathcal{N}(\bx\condon\bm{\mu}_2(\by), \bm{I})$, where ${\bm{\mu}_1(\by) = (\sin{(\pi y_1)} + 0.5, \cos{(\pi y_2) - 0.5})}$ and ${\bm{\mu}_2(\by) = (\cos{(2\pi y_2)} - 0.5, \sin{(0.5\pi y_1) + 0.5})}$. 
We draw a set of conditioning inputs~$\{\by^{(n)}\}_{n=1}^{N}$ with~$N=4000$ from a uniform distribution over an annulus with inner radius \num{0.75} and outer radius \num{1.25} and sample~$\bx^{(n)}\sim q(\bx\condon\by^{(n)})$.
We parameterize the velocity field using an MLP and train an ensemble of~$M=5$ FM models with a batch size of \num{256} for \num{300} epochs using the standard loss~\eqref{eq:pre_fm_loss_conditional}.
As shown in~\Cref{fig:toy_example}, our VFD score assigns regions far from the training data a high level of epistemic uncertainty, which closely resembles the average KL divergence~$\frac{1}{M}\sum_{i=1}^M\kl(p^{\bth_i}\,||\,q)$ between the learned distributions and the ground truth.

\begin{figure}[!htb]
    \centering
    \includegraphics[width=\linewidth]{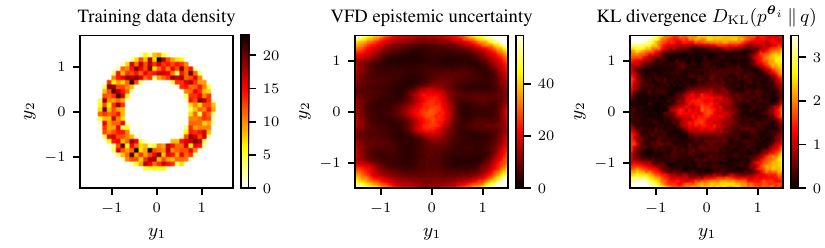}
    \caption{\textbf{VFD for a 2D toy example.} We visualize epistemic uncertainty estimation for a 2D generative modeling problem. Our VFD uncertainty score is high for inputs far from the training distribution, similar to the KL divergence between the learned models' conditional distributions and the ground truth.}
    \label{fig:toy_example}
\end{figure}

\end{example}

\subsection{Calibration}

\Cref{tab:calibration_full} is an extension of~\Cref{tab:calibration_summary} and also reports the Pearson correlation coefficient, measuring the strength of the linear relationship between per-task uncertainties and success rates.

\begin{table}[!htb]
    \setlength{\tabcolsep}{2.7pt}
    \centering
    \small
    \caption{\textbf{Calibration analysis -- extended.} Negative Spearman's rank correlation ($\uparrow$) and negative Pearson correlation ($\uparrow$) between uncertainty estimates and per-task success rates in~\liberoten, averaged across iterative fine-tuning rounds. VFD offers superior calibration compared to other uncertainty estimation approaches.}
    \begin{tabular}{l l c c c c c c c}
        \toprule
        Policy & Metric & Action-L2 & ACE & DECU & GU & Entropy & Perplexity & \textbf{VFD (ours)} \\
        \midrule
        \multirow{2}{*}{SmolVLA}  
            & $-$Spearman & \num{0.50}$^{\pm \num{0.13}}$ & \num{0.31}$^{\pm \num{0.12}}$ & \num{0.31}$^{\pm \num{0.13}}$ & \num{0.62}$^{\pm \num{0.00}}$ & \num{0.10}$^{\pm \num{0.12}}$ & \num{-0.04}$^{\pm \num{0.09}}$ & $\mathbf{0.71}^{\pm \num{0.03}}$ \\
            & $-$Pearson  & \num{0.48}$^{\pm \num{0.09}}$ & \num{0.36}$^{\pm \num{0.08}}$ & \num{0.23}$^{\pm \num{0.15}}$ & \num{0.65}$^{\pm \num{0.02}}$ & \num{0.23}$^{\pm \num{0.21}}$ & \num{0.02}$^{\pm \num{0.15}}$ & $\mathbf{0.71}^{\pm \num{0.02}}$ \\
        \midrule
        \multirow{2}{*}{X-VLA}  
            & $-$Spearman & \num{0.33}$^{\pm \num{0.24}}$ & \num{0.14}$^{\pm \num{0.05}}$ & \num{0.06}$^{\pm \num{0.04}}$ & \num{0.48}$^{\pm \num{0.06}}$ & N/A & N/A & $\mathbf{0.53}^{\pm \num{0.02}}$ \\
            & $-$Pearson  & \num{0.33}$^{\pm \num{0.22}}$ & \num{0.15}$^{\pm \num{0.07}}$ & \num{0.06}$^{\pm \num{0.05}}$ & \num{0.45}$^{\pm \num{0.13}}$ & N/A & N/A & $\mathbf{0.48}^{\pm \num{0.08}}$ \\
        \midrule
        \multirow{2}{*}{FastWAM}  
            & $-$Spearman & \num{-0.12}$^{\pm \num{0.06}}$ & \num{-0.29}$^{\pm \num{0.07}}$ & \num{0.03}$^{\pm \num{0.07}}$ & \num{0.27}$^{\pm \num{0.04}}$ & N/A & N/A & $\mathbf{0.44}^{\pm
 \num{0.17}}$ \\
            & $-$Pearson  & \num{-0.03}$^{\pm \num{0.14}}$ & \num{-0.20}$^{\pm \num{0.07}}$ & \num{0.06}$^{\pm \num{0.05}}$ & \num{0.29}$^{\pm \num{0.06}}$ & N/A & N/A & $\mathbf{0.42}^{\pm
 \num{0.12}}$ \\
        \bottomrule
\end{tabular}
    \label{tab:calibration_full}
\end{table}

\subsubsection{Ensemble Size Ablation}
Since fine-tuning generalist policies is expensive, keeping the ensemble small (or avoiding it entirely, cf.~\Cref{tab:computational_cost}) is highly desirable.
\Cref{fig:ensemble_size_ablation} shows the impact of the ensemble size~$M$ on calibration for all fine-tuning rounds of SmolVLA on~\liberoten.
The VLA ensembles consistently remain well-calibrated even for~$M=2$, indicating that they do not collapse to a single mode of the posterior. 
Due to the competitive performance and computational advantage of two-member ensembles, we adopt this ensemble size in the experiments unless stated otherwise. \looseness -1

\begin{figure}[!htb]
    \centering
    \includegraphics[width=0.9\linewidth]{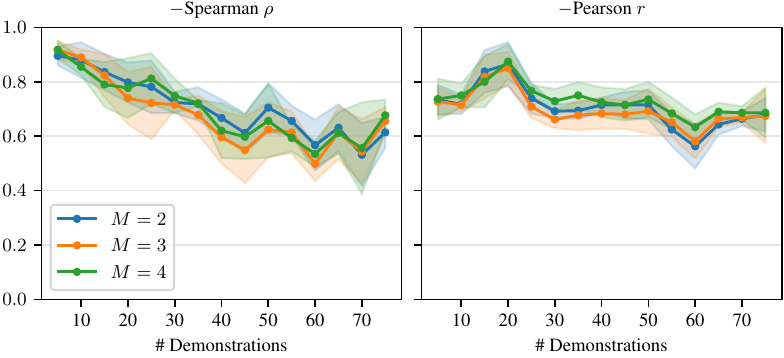}
    \caption{\textbf{Ensemble size ablation.} During active fine-tuning, the number of ensemble models has little effect on the calibration of VFD.}
    \label{fig:ensemble_size_ablation}
\end{figure}

\subsubsection{Ensembling vs. Laplace Approximation}

We compare training an ensemble to computing the Laplace approximation~\citep{daxberger2021laplace} for sampling from the model posterior.
Since computing the full Laplace approximation for the entire SmolVLA model~(550M) is computationally intractable, we resort to the last-layer Laplace approximation, which is also used by~\citet{jazbec2025generative} to compute GU in the original work.
We fit a diagonal last-layer Laplace approximation with \texttt{laplace-torch}~\citep{daxberger2021laplace} around each trained SmolVLA checkpoint. Concretely, we freeze the full model and place the posterior only on the final action-projection layer, which maps the action expert's hidden states to the predicted FM action velocity.
The Laplace posterior is fit as a regression model on the same training episodes used for that active-learning round: for each calibration frame, we sample FM time~$s$, noise~$\bx_0\sim\gaussian$, and target action~$\bx_1$, and use the corresponding velocity~$\bx_1 - \bx_0$ as the regression target. We use all selected training episodes for the round, a calibration fraction of~\num{1.0}, and a diagonal Hessian approximation.

Even when limited to the last layer, the computational cost of the last-layer Laplace approximation remains considerable, requiring about \num{5}~hours per checkpoint per round on an NVIDIA RTX 4090 GPU. This is much higher than the \num{30}~minutes it takes to fine-tune a second ensemble member for~\num{4000}~steps on the same GPU.
Due to the high cost, we only fit the Laplace approximation for the first \num{5}~rounds. As shown in~\Cref{fig:ensemble_vs_laplace}, the two-member ensemble provides much better VFD score calibration than the Laplace approximation despite being faster to train. 

\begin{figure}[tb!]
    \centering
    \includegraphics[width=0.9\linewidth]{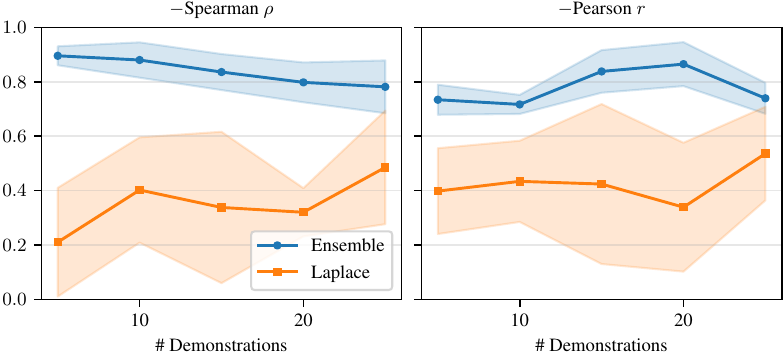}
    \caption{\textbf{Ensembling vs. Laplace Approximation.} Calibration is measured by the negative Spearman rank correlation (left) and the negative Pearson correlation (right) between the average task success rate and VFD uncertainty. Two separate ensemble members achieve much better calibration than sampling a second model from the last-layer Laplace approximation.}
    \label{fig:ensemble_vs_laplace}
\end{figure}

\subsection{Learning Curves of Active Fine-Tuning}

\begin{table}[tb!]
    \centering
    \small
    \setlength{\tabcolsep}{4pt}
    \caption{\textbf{Statistical significance on \pusht.} Mean and sample standard deviation over ten paired training seeds.}
    \label{tab:significance}
    \begin{tabular}{l c c c c c c}
        \toprule
        Metric & \method w/ VFD & \method w/ GU & Difference & VFD better & $p_\text{Wilcoxon}$ & $p_{t\text{-test}}$ \\
        \midrule
        AULC      & $46.6^{\pm 2.5}$ & $44.7^{\pm 2.9}$ & $+1.9$~$[+0.8, +3.2]$ & $9/10$ & $0.006$ & $0.010$ \\
        Last-3 SR & $59.3^{\pm 3.1}$ & $54.3^{\pm 3.3}$ & $+5.0$~$[+2.8, +7.1]$ & $9/10$ & $0.004$ & $0.001$ \\
        \bottomrule
    \end{tabular}
\end{table}

\begin{table}[tb!]
    \setlength{\tabcolsep}{4.5pt}
    \small
    \centering
    \caption{\textbf{Sample efficiency.} Number of demonstrations to reach certain success rates (SR), area under the learning curve (AULC) ($\uparrow$), and Last-3 SR ($\uparrow$) for different active fine-tuning strategies. We report the six highest SR thresholds that are multiples of~\qty{5}{\percent} and reached by at least one method. Thresholds not reached within the maximum demonstration budget are marked as ``N/R''.}
    \begin{tabular}{l c c c c c}
    \toprule
    \multicolumn{6}{c}{\pusht: FM-Policy} \\
    \midrule
    SR Threshold & Random & Diversity & \method w/ Action-L2 & \method w/ GU & \textbf{\method w/ VFD} \\
    \midrule
    $\geq \qty{30}{\percent}$ & $30.0^{\pm 15.3}$ & $21.1^{\pm 11.4}$ & $17.5^{\pm 13.1}$ & $15.6^{\pm 8.8}$ & $\mathbf{15.0}^{\pm 8.2}$ \\
    $\geq \qty{35}{\percent}$ & $47.5^{\pm 27.7}$ & $36.7^{\pm 11.5}$ & $29.4^{\pm 17.4}$ & $29.4^{\pm 12.6}$ & $\mathbf{28.5}^{\pm 13.1}$ \\
    $\geq \qty{40}{\percent}$ & $71.5^{\pm 21.9}$ & $51.7^{\pm 16.0}$ & $48.1^{\pm 18.9}$ & $\mathbf{38.9}^{\pm 12.2}$ & $43.0^{\pm 8.9}$ \\
    $\geq \qty{45}{\percent}$ & $98.5^{\pm 21.0}$ & $86.7^{\pm 22.4}$ & $80.0^{\pm 38.7}$ & $61.7^{\pm 13.9}$ & $\mathbf{58.0}^{\pm 14.0}$ \\
    $\geq \qty{50}{\percent}$ & N/R & $109.4^{\pm 20.4}$ & N/R & $83.3^{\pm 12.2}$ & $\mathbf{78.5}^{\pm 16.3}$ \\
    $\geq \qty{55}{\percent}$ & N/R & N/R & N/R & $105.0^{\pm 24.2}$ & $\mathbf{96.0}^{\pm 16.3}$ \\
    \midrule
    AULC & $38.0^{\pm 3.2}$ & $40.5^{\pm 1.9}$ & $39.8^{\pm 3.4}$ & $44.7^{\pm 2.9}$ & $\mathbf{46.6}^{\pm 2.5}$ \\
    Last-3 SR & $46.3^{\pm 3.1}$ & $51.2^{\pm 3.2}$ & $45.0^{\pm 2.6}$ & $54.3^{\pm 3.3}$ & $\mathbf{59.3}^{\pm 3.1}$ \\
    \bottomrule
\end{tabular} \\
\vspace{2mm}
\begin{tabular}{l c c c c c c}
    \toprule
    \multicolumn{7}{c}{\liberoten: SmolVLA} \\
    \midrule
    SR Threshold & Random & Diversity & AMF & \method w/ Action-L2 & \method w/ GU & \textbf{\method w/ VFD} \\
    \midrule
    $\geq \qty{40}{\percent}$ & $28.3^{\pm 7.6}$ & $\mathbf{23.3}^{\pm 2.9}$ & $30.0^{\pm 5.0}$ & $25.0^{\pm 5.0}$ & $25.0^{\pm 8.7}$ & $25.0^{\pm 5.0}$ \\
    $\geq \qty{45}{\percent}$ & $33.3^{\pm 7.6}$ & $36.7^{\pm 5.8}$ & $38.3^{\pm 5.8}$ & $30.0^{\pm 8.7}$ & $28.3^{\pm 14.4}$ & $\mathbf{26.7}^{\pm 5.8}$ \\
    $\geq \qty{50}{\percent}$ & $43.3^{\pm 12.6}$ & $48.3^{\pm 10.4}$ & $46.7^{\pm 7.6}$ & $45.0^{\pm 15.0}$ & $36.7^{\pm 11.5}$ & $\mathbf{30.0}^{\pm 5.0}$ \\
    $\geq \qty{55}{\percent}$ & N/R & N/R & $53.3^{\pm 2.9}$ & $45.0^{\pm 7.1}$ & $55.0^{\pm 10.0}$ & $\mathbf{40.0}^{\pm 10.0}$ \\
    $\geq \qty{60}{\percent}$ & N/R & N/R & $66.7^{\pm 10.4}$ & N/R & $63.3^{\pm 10.4}$ & $\mathbf{50.0}^{\pm 5.0}$ \\
    $\geq \qty{65}{\percent}$ & N/R & N/R & N/R & N/R & N/R & $\mathbf{62.5}^{\pm 3.5}$ \\
    \midrule
    AULC & $44.6^{\pm 2.2}$ & $45.5^{\pm 1.8}$ & $45.6^{\pm 3.3}$ & $47.2^{\pm 5.1}$ & $48.6^{\pm 5.0}$ & $\mathbf{52.3}^{\pm 2.2}$ \\
    Last-3 SR & $52.6^{\pm 3.3}$ & $53.1^{\pm 2.0}$ & $59.4^{\pm 3.3}$ & $55.5^{\pm 6.8}$ & $61.0^{\pm 2.9}$ & $\mathbf{66.1}^{\pm 2.8}$ \\
    \bottomrule
\end{tabular} \\
\vspace{2mm}
\begin{tabular}{l c c c c c c}
    \toprule
    \multicolumn{7}{c}{\liberoten: X-VLA} \\
    \midrule
    SR Threshold & Random & Diversity & AMF & \method w/ Action-L2 & \method w/ GU & \textbf{\method w/ VFD} \\
    \midrule
    $\geq \qty{25}{\percent}$ & $18.3^{\pm 2.9}$ & $15.0^{\pm 0.0}$ & $18.3^{\pm 2.9}$ & $16.7^{\pm 2.9}$ & $13.3^{\pm 2.9}$ & $\mathbf{11.7}^{\pm 2.9}$ \\
    $\geq \qty{30}{\percent}$ & $21.7^{\pm 7.6}$ & $20.0^{\pm 0.0}$ & $23.3^{\pm 7.6}$ & $21.7^{\pm 2.9}$ & $16.7^{\pm 7.6}$ & $\mathbf{15.0}^{\pm 5.0}$ \\
    $\geq \qty{35}{\percent}$ & $26.7^{\pm 7.6}$ & $\mathbf{21.7}^{\pm 2.9}$ & $30.0^{\pm 5.0}$ & $25.0^{\pm 0.0}$ & $26.7^{\pm 5.8}$ & $23.3^{\pm 2.9}$ \\
    $\geq \qty{40}{\percent}$ & $40.0^{\pm 0.0}$ & $\mathbf{33.3}^{\pm 7.6}$ & $36.7^{\pm 2.9}$ & $35.0^{\pm 5.0}$ & $35.0^{\pm 10.0}$ & $\mathbf{33.3}^{\pm 5.8}$ \\
    $\geq \qty{45}{\percent}$ & $50.0^{\pm 13.2}$ & $50.0^{\pm 15.0}$ & $51.7^{\pm 2.9}$ & $56.7^{\pm 7.6}$ & $51.7^{\pm 10.4}$ & $\mathbf{43.3}^{\pm 2.9}$ \\
    $\geq \qty{50}{\percent}$ & N/R & N/R & $65.0^{\pm 10.0}$ & $67.5^{\pm 3.5}$ & $57.5^{\pm 17.7}$ & $\mathbf{55.0}^{\pm 0.0}$ \\
    \midrule
    AULC & $37.7^{\pm 1.2}$ & $40.2^{\pm 1.9}$ & $39.5^{\pm 1.9}$ & $39.2^{\pm 0.9}$ & $40.7^{\pm 3.5}$ & $\mathbf{42.5}^{\pm 1.0}$ \\
    Last-3 SR & $45.1^{\pm 0.3}$ & $49.0^{\pm 3.3}$ & $51.0^{\pm 1.3}$ & $50.9^{\pm 2.2}$ & $49.2^{\pm 1.8}$ & $\mathbf{52.9}^{\pm 2.9}$ \\
    \bottomrule
\end{tabular} \\
\vspace{2mm}
\begin{tabular}{l c c c c c c}
    \toprule
    \multicolumn{7}{c}{\liberoten: FastWAM} \\
    \midrule
    SR Threshold & Random & Diversity & AMF & \method w/ Action-L2 & \method w/ GU & \textbf{\method w/ VFD} \\
    \midrule
    $\geq \qty{45}{\percent}$ & $20.0^{\pm 5.0}$ & $21.7^{\pm 2.9}$ & $25.0^{\pm 0.0}$ & $20.0^{\pm 5.0}$ & $\mathbf{16.7}^{\pm 2.9}$ & $18.3^{\pm 7.6}$ \\
    $\geq \qty{50}{\percent}$ & $26.7^{\pm 2.9}$ & $30.0^{\pm 0.0}$ & $28.3^{\pm 2.9}$ & $26.7^{\pm 2.9}$ & $26.7^{\pm 5.8}$ & $\mathbf{21.7}^{\pm 7.6}$ \\
    $\geq \qty{55}{\percent}$ & $38.3^{\pm 2.9}$ & $33.3^{\pm 2.9}$ & $33.3^{\pm 2.9}$ & $31.7^{\pm 2.9}$ & $\mathbf{28.3}^{\pm 2.9}$ & $\mathbf{28.3}^{\pm 10.4}$ \\
    $\geq \qty{60}{\percent}$ & $53.3^{\pm 18.9}$ & $56.7^{\pm 10.4}$ & $38.3^{\pm 2.9}$ & $40.0^{\pm 8.7}$ & $35.0^{\pm 5.0}$ & $\mathbf{33.3}^{\pm 5.8}$ \\
    $\geq \qty{65}{\percent}$ & $62.5^{\pm 17.7}$ & $75.0^{\pm 0.0}$ & $48.3^{\pm 15.3}$ & $58.3^{\pm 2.9}$ & $\mathbf{41.7}^{\pm 5.8}$ & $\mathbf{41.7}^{\pm 7.6}$ \\
    $\geq \qty{70}{\percent}$ & N/R & N/R & N/R & $60.0^{\pm 0.0}$ & $57.5^{\pm 24.7}$ & $\mathbf{56.7}^{\pm 5.8}$ \\
    \midrule
    AULC & $54.2^{\pm 2.5}$ & $53.9^{\pm 0.8}$ & $55.6^{\pm 1.4}$ & $57.0^{\pm 3.3}$ & $59.0^{\pm 4.4}$ & $\mathbf{61.1}^{\pm 3.5}$ \\
    Last-3 SR & $64.8^{\pm 6.6}$ & $63.9^{\pm 0.3}$ & $69.4^{\pm 1.3}$ & $71.6^{\pm 5.2}$ & $70.6^{\pm 6.1}$ & $\mathbf{74.8}^{\pm 3.3}$ \\
    \bottomrule
\end{tabular}
    \label{tab:active_learning_comparison}
\end{table}

\begin{figure}[tb!]
    \centering
    \includegraphics[width=\linewidth]{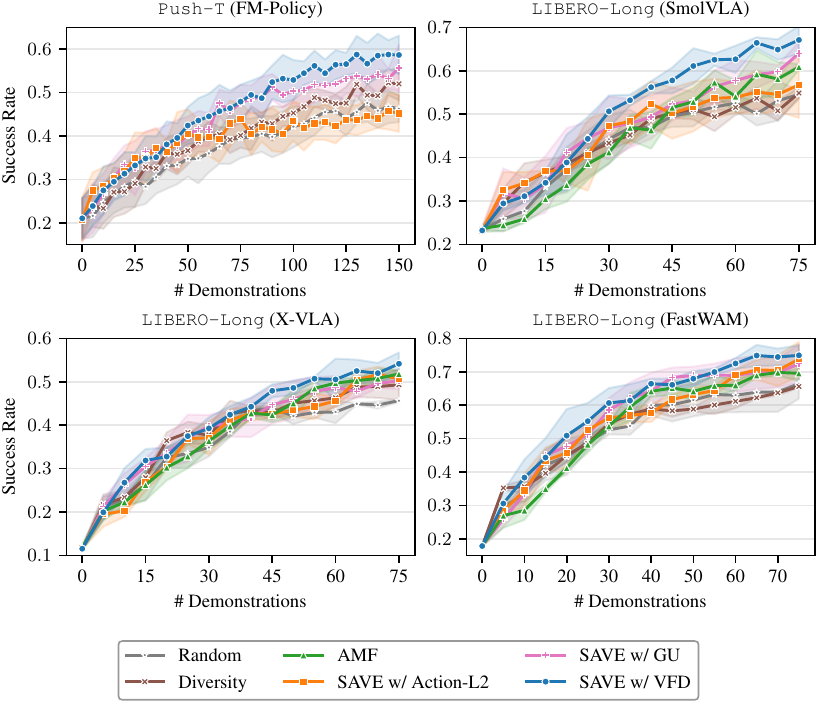}
    \caption{\textbf{Active learning curves.} These correspond to the active fine-tuning results reported in Tables~\ref{tab:active_learning_comparison_summary} and~\ref{tab:active_learning_comparison}. \method w/ VFD consistently yields the largest performance gains with the available demonstration budget.}
    \label{fig:al_success_rate_comparison}
\end{figure}

\begin{figure}
    \centering
    \includegraphics[width=\linewidth]{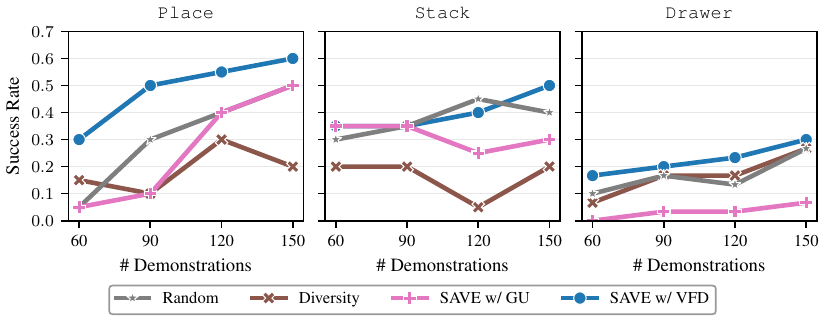}
    \caption{\textbf{Real-world active fine-tuning results}. 
    SAVE w/ VFD consistently improves across all tasks, whereas Diversity and SAVE w/ GU stagnate on certain tasks. Notably, random data acquisition is the second-best method on hardware, highlighting the difficulty of effectively allocating demonstration budgets for multitask improvement. These per-task learning curves constitute the average shown in~\Cref{fig:real_world_success_rates}.}
    \label{fig:real_world_al_full}
\end{figure}

\begin{table}[!tb]    
    \centering
    \small
    \caption{{\textbf{Hyperparameter selection for the uncertainty-guided active fine-tuning methods.}} Success rate~(SR) in percent ($\uparrow$) after different active fine-tuning rounds for SmolVLA on \liberoten. The settings with the highest final success rate, reported in Tables~\ref{tab:active_learning_comparison_summary} and~\ref{tab:active_learning_comparison}, are marked with an asterisk (*).} 
    \begin{tabular}{l l c c c}
        \toprule
        Method & Hyperparameter & Round-5 SR & Round-10 SR & Round-15 SR \\
        \midrule
        \multirow{4}{*}{AMF} & $\sigma=10^{-4}$ & $40.4^{\pm 4.1}$ & $52.2^{\pm 4.4}$ & $58.1^{\pm 3.1}$ \\
         & $\sigma=10^{-3}$ & $40.2^{\pm 4.0}$ & $53.3^{\pm 3.7}$ & $59.7^{\pm 6.7}$ \\
         & $\sigma=10^{-2}$* & $38.6^{\pm 4.4}$ & $52.9^{\pm 3.9}$ & $60.9^{\pm 0.8}$ \\
         & $\sigma=10^{-1}$ & $39.6^{\pm 5.7}$ & $52.3^{\pm 4.8}$ & $57.6^{\pm 2.8}$ \\
        \midrule
        \multirow{5}{*}{\method w/ Action-L2} & $\tau = 0$ & $40.6^{\pm 3.0}$ & $51.1^{\pm 4.2}$ & $54.0^{\pm 4.2}$ \\
         & $\tau = 1$ & $41.4^{\pm 5.2}$ & $46.9^{\pm 6.4}$ & $52.2^{\pm 2.2}$ \\
         & $\tau = 1.5$* & $40.7^{\pm 10.1}$ & $51.8^{\pm 7.0}$ & $56.8^{\pm 9.3}$ \\
         & $\tau = 2$ & $40.7^{\pm 11.0}$ & $53.0^{\pm 7.0}$ & $55.0^{\pm 3.2}$ \\
         & $\tau = 2.5$ & $41.8^{\pm 7.0}$ & $49.0^{\pm 9.0}$ & $53.4^{\pm 6.9}$ \\
        \midrule
        \multirow{5}{*}{\method w/ GU} & $\tau = 0$ & $40.3^{\pm 5.5}$ & $49.6^{\pm 7.0}$ & $53.6^{\pm 3.0}$ \\
         & $\tau = 1$* & $44.4^{\pm 4.6}$ & $53.0^{\pm 2.7}$ & $64.0^{\pm 3.2}$ \\
         & $\tau = 1.5$ & $42.0^{\pm 1.5}$ & $57.4^{\pm 5.3}$ & $59.2^{\pm 6.7}$ \\
         & $\tau = 2$ & $45.6^{\pm 3.7}$ & $57.2^{\pm 4.4}$ & $60.6^{\pm 3.5}$ \\
         & $\tau = 2.5$ & $46.7^{\pm 4.7}$ & $59.3^{\pm 4.2}$ & $60.1^{\pm 2.7}$ \\
        \midrule
        \multirow{5}{*}{\method w/ VFD} & $\tau = 0$ & $37.7^{\pm 1.2}$ & $47.1^{\pm 1.0}$ & $58.0^{\pm 7.2}$ \\
         & $\tau = 1$ & $44.8^{\pm 6.6}$ & $54.8^{\pm 3.3}$ & $59.2^{\pm 4.8}$ \\
         & $\tau = 1.5$ & $46.6^{\pm 2.0}$ & $53.8^{\pm 1.7}$ & $60.1^{\pm 4.6}$ \\
         & $\tau = 2$ & $46.0^{\pm 6.7}$ & $59.0^{\pm 2.6}$ & $64.2^{\pm 3.7}$ \\
         & $\tau = 2.5$* & $44.3^{\pm 6.2}$ & $61.1^{\pm 5.0}$ & $67.1^{\pm 4.0}$ \\
        \bottomrule
\end{tabular}
    \label{tab:active_learning_comparison_full}
\end{table}

In~\Cref{fig:al_success_rate_comparison}, we provide the full learning curves in our simulation experiments, corresponding to the results reported in Tables~\ref{tab:active_learning_comparison_summary} and~\ref{tab:active_learning_comparison}.
\Cref{fig:real_world_al_full} shows the active fine-tuning performance for each of the three real-world tasks.

\subsection{Statistical Significance of Uncertainty-Guided Data Acquisition}
\label{app:significance}
The active fine-tuning experiments with FM-Policy on~\pusht require far fewer computational resources than those with the three generalist policies.
Therefore, we can repeat experiments across ten seeds in this environment to conduct a statistical test.
Specifically, we assess whether the advantage of VFD over the strongest uncertainty-based baseline, GU, observed in all environments, is statistically robust.
We compare the two methods in a paired manner: For each seed, both start from the same pre-trained ensemble and use the same training seed.
We test the one-sided hypothesis that \method w/ VFD achieves higher AULC and Last-3 SR than \method w/ GU using the Wilcoxon signed-rank test and the paired~$t$-test.
As shown in~\Cref{tab:significance}, \method w/ VFD outperforms \method w/ GU in nine out of ten seeds on both metrics, and the improvement is significant at the \qty{1}{\percent} level for both metrics under the Wilcoxon test. 
These conclusions also hold for two-sided tests at a \qty{2}{\percent} level and after a Holm correction for testing two metrics.

\subsection{Iterative Fine-Tuning vs. Retraining}
We also test whether our iterative active-learning loop, designed to minimize data-collection effort, requires an additional separate full fine-tuning stage to yield optimal policy performance.
To this end, we compare the final-round SmolVLA policy~$\pi_{15}$ to a policy retrained from the base checkpoint using a \num{50}/\num{50} split between the ensemble training data~$\mathcal{D}_{\text{pre}}$ and the actively selected data~$\mathcal{D}_{\text{new}}^{(\leq15)}$.
When using the demonstrations selected at the best-performing temperature $\tau=2.5$, the retrained policy achieves an overall success rate of~$70.6\%^{\pm 0.9}$, which is only~\qty{3.5}{}\,pp higher than that of the iteratively fine-tuned policy.
This demonstrates that uncertainty-guided acquisition produces a strong policy already during the iterative active learning process.

\subsection{Hyperparameter Selection for Uncertainty-Guided Active Fine-Tuning}
\Cref{fig:success_and_uncertainty_share} (left) shows that most of the gain of \method w/ VFD over random acquisition stems from the prioritization at the task level, while uncertainty-based initial observation selection yields a minor additional improvement. 
A key reason for this performance gain is shown in the right of~\Cref{fig:success_and_uncertainty_share} and in~\Cref{fig:selection_uncertainty}.
While pre-trained policies naturally exhibit varying levels of uncertainty across tasks, our uncertainty-guided active fine-tuning strategy can reduce these differences and improve multitask performance more quickly by allocating the demonstration budget to tasks that require it most.
\Cref{fig:exploration-success-pareto} complements this observation, showing a detailed view of the trade-off between exploiting the uncertainty estimates and exploring the task space evenly. 
The diversity of the categorical task distribution can be measured by its entropy. 
Over multiple rounds, \method tends towards high-entropy categorical distributions over tasks (i.e., more diversity in the selection). We observe a Pareto front: \method w/ VFD can trade off this diversity (lower exploration) for faster improvement (higher exploitation) in multitask success by focusing on high-uncertainty tasks (cf.~\Cref{fig:exploration-success-pareto}).

\begin{figure}
    \centering
    \includegraphics[width=\linewidth]{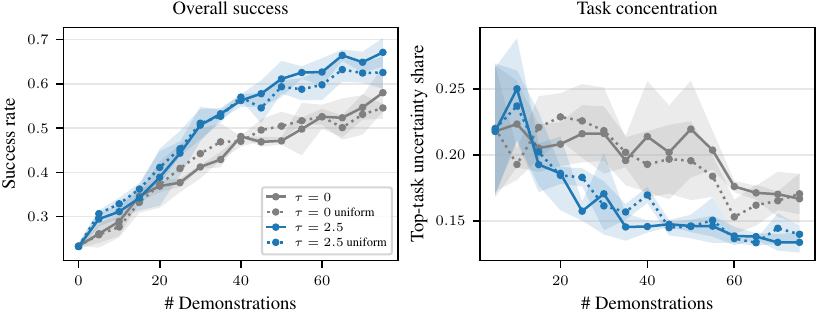}
    \caption{\textbf{Effect of the task-sampling temperature $\tau$ on \method}. Larger $\tau$ biases expert queries toward higher-uncertainty tasks, with uniform sampling for $\tau=0$. \textbf{Left:} 
    Uncertainty guidance is beneficial both for sampling tasks and initial observations.
    Legend entries containing \textit{uniform} correspond to sampling initial observations within a sampled task uniformly instead of uncertainty-guided. \textbf{Right:} Uncertainty-based sampling more rapidly reduces the fraction of uncertainty concentrated in the most uncertain task, indicating better allocation of demonstrations to underperforming tasks. 
    }
    \label{fig:success_and_uncertainty_share}
\end{figure}


\begin{figure}[!htb]
    \centering
    \includegraphics[width=0.95\textwidth]{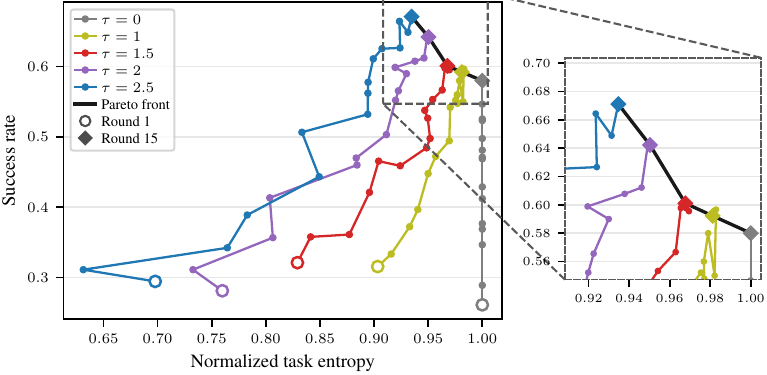}
    \caption{\textbf{Pareto front of the exploration-exploitation trade-off.} \method reduces the difference in prior knowledge about tasks; for $\tau \leq 2.5$, the temperature parameter controls an exploration–exploitation trade-off between task coverage and final success rate.}
    \label{fig:exploration-success-pareto}
\end{figure}

\begin{figure}
    \centering
    \includegraphics[width=\linewidth]{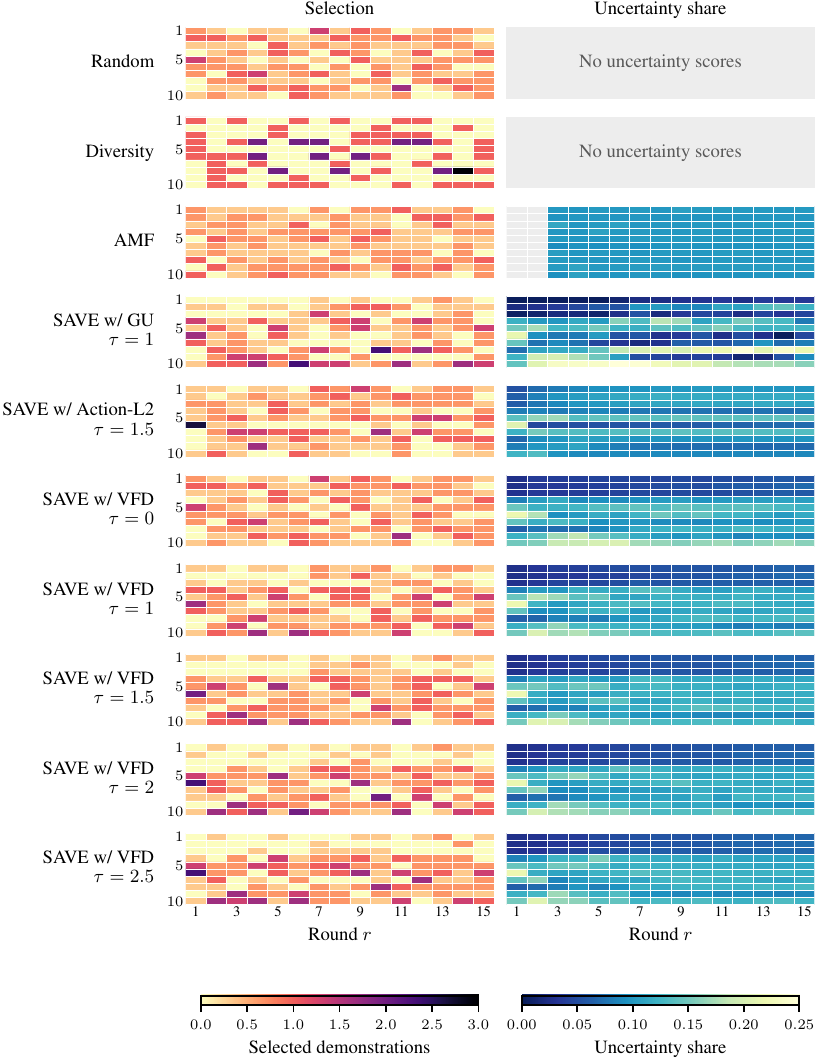}
    \caption{\textbf{Task selection}. Behavior of uncertainty-guided data acquisition in active fine-tuning for different temperature values~$\tau$ over all active fine-tuning rounds with SmolVLA on \liberoten. Left: Number of selected episodes per task. Right: Relative share of overall uncertainty per task.}
    \label{fig:selection_uncertainty}
\end{figure}

\subsection{Failure Detection}

We provide the full results of our failure detection experiments in Tables~\ref{tab:failure_prediction_pusht}, \ref{tab:failure_prediction_liberoplus} and~\ref{tab:failure_prediction_real_world}.

\begin{table}[!htb]
    \centering
    \small
    \caption{\textbf{Failure detection results in \pusht.} Mean and standard deviation over ten training seeds.}
    \begin{tabular}{l l c c c c c}
        \toprule
        Metric & STAC & logpZO & RND-OE & ACE & \textbf{VFD (ours)} \\
        \midrule
        Acc. $\uparrow$ & $\underline{0.55}^{\pm 0.00}$ & $0.53^{\pm 0.01}$ & $0.50^{\pm 0.02}$ & $\underline{0.55}^{\pm 0.06}$ & $\mathbf{0.64}^{\pm 0.05}$ \\
        TWA $\uparrow$ & $\underline{0.51}^{\pm 0.00}$ & $\underline{0.51}^{\pm 0.02}$ & $0.50^{\pm 0.02}$ & $\underline{0.51}^{\pm 0.03}$ & $\mathbf{0.56}^{\pm 0.05}$ \\
        DT $\downarrow$ & $0.56^{\pm 0.04}$ & $0.23^{\pm 0.09}$ & $0.00^{\pm 0.00}$ & $0.30^{\pm 0.04}$ & $0.41^{\pm 0.05}$ \\
        TPR $\uparrow$ & $0.13^{\pm 0.02}$ & $0.18^{\pm 0.09}$ & $0.15^{\pm 0.08}$ & $0.23^{\pm 0.15}$ & $0.38^{\pm 0.08}$ \\
        TNR $\uparrow$ & $0.97^{\pm 0.02}$ & $0.87^{\pm 0.07}$ & $0.86^{\pm 0.12}$ & $0.86^{\pm 0.05}$ & $0.89^{\pm 0.07}$ \\
        \bottomrule
    \end{tabular}
    \label{tab:failure_prediction_pusht}
\end{table}

\begin{table}[!htb]
    \centering
    \small
    \caption{\textbf{Failure detection on \liberoplus.} We group the results by perturbation family and provide the mean and standard deviation over three training seeds. 
    }
    \begin{tabular}{l l c c c c c}
        \toprule
        Perturbation & Metric & STAC & logpZO & RND-OE & ACE & \textbf{VFD (ours)} \\
        \midrule
        \multirow{5}{*}{Camera}
            & Acc. $\uparrow$ & $0.68^{\pm 0.06}$ & $0.65^{\pm 0.02}$ & $0.61^{\pm 0.02}$ & $\mathbf{0.73}^{\pm 0.00}$ & $\underline{0.72}^{\pm 0.05}$ \\
            & TWA $\uparrow$ & $0.56^{\pm 0.04}$ & $\mathbf{0.58}^{\pm 0.03}$ & $\underline{0.57}^{\pm 0.02}$ & $0.56^{\pm 0.01}$ & $0.56^{\pm 0.06}$ \\
            & DT $\downarrow$ & $\mathbf{0.32}^{\pm 0.02}$ & $0.15^{\pm 0.02}$ & $0.08^{\pm 0.01}$ & $0.50^{\pm 0.06}$ & $\underline{0.37}^{\pm 0.03}$ \\
            & TPR $\uparrow$ & $0.72^{\pm 0.12}$ & $0.99^{\pm 0.01}$ & $1.00^{\pm 0.00}$ & $0.73^{\pm 0.09}$ & $0.90^{\pm 0.04}$ \\
            & TNR $\uparrow$ & $0.63^{\pm 0.04}$ & $0.32^{\pm 0.04}$ & $0.22^{\pm 0.04}$ & $0.74^{\pm 0.08}$ & $0.54^{\pm 0.12}$ \\
        \midrule
        \multirow{5}{*}{Robot init.}
            & Acc. $\uparrow$ & $0.70^{\pm 0.02}$ & $0.74^{\pm 0.02}$ & $0.72^{\pm 0.02}$ & $\underline{0.76}^{\pm 0.04}$ & $\mathbf{0.78}^{\pm 0.03}$ \\
            & TWA $\uparrow$ & $0.61^{\pm 0.01}$ & $\underline{0.68}^{\pm 0.03}$ & $\mathbf{0.69}^{\pm 0.03}$ & $0.64^{\pm 0.05}$ & $\underline{0.68}^{\pm 0.04}$ \\
            & DT $\downarrow$ & $0.30^{\pm 0.02}$ & $\underline{0.11}^{\pm 0.03}$ & $\mathbf{0.07}^{\pm 0.01}$ & $0.39^{\pm 0.05}$ & $0.22^{\pm 0.03}$ \\
            & TPR $\uparrow$ & $0.68^{\pm 0.07}$ & $0.99^{\pm 0.02}$ & $0.99^{\pm 0.01}$ & $0.72^{\pm 0.10}$ & $0.91^{\pm 0.04}$ \\
            & TNR $\uparrow$ & $0.73^{\pm 0.04}$ & $0.49^{\pm 0.02}$ & $0.46^{\pm 0.04}$ & $0.80^{\pm 0.06}$ & $0.65^{\pm 0.02}$ \\
        \midrule
        \multirow{5}{*}{Sensor noise}
            & Acc. $\uparrow$ & $0.67^{\pm 0.06}$ & $0.66^{\pm 0.04}$ & $0.66^{\pm 0.02}$ & $\underline{0.74}^{\pm 0.02}$ & $\mathbf{0.80}^{\pm 0.02}$ \\
            & TWA $\uparrow$ & $0.56^{\pm 0.05}$ & $0.61^{\pm 0.03}$ & $\underline{0.63}^{\pm 0.02}$ & $0.60^{\pm 0.01}$ & $\mathbf{0.67}^{\pm 0.02}$ \\
            & DT $\downarrow$ & $0.41^{\pm 0.05}$ & $0.09^{\pm 0.03}$ & $0.07^{\pm 0.01}$ & $\underline{0.39}^{\pm 0.12}$ & $\mathbf{0.30}^{\pm 0.02}$ \\
            & TPR $\uparrow$ & $0.63^{\pm 0.10}$ & $1.00^{\pm 0.00}$ & $1.00^{\pm 0.00}$ & $0.81^{\pm 0.09}$ & $0.94^{\pm 0.03}$ \\
            & TNR $\uparrow$ & $0.71^{\pm 0.04}$ & $0.31^{\pm 0.08}$ & $0.33^{\pm 0.05}$ & $0.67^{\pm 0.12}$ & $0.66^{\pm 0.02}$ \\
        \midrule
        \multirow{5}{*}{Background}
            & Acc. $\uparrow$ & $0.68^{\pm 0.04}$ & $0.66^{\pm 0.00}$ & $0.66^{\pm 0.00}$ & $\underline{0.74}^{\pm 0.03}$ & $\mathbf{0.77}^{\pm 0.02}$ \\
            & TWA $\uparrow$ & $0.57^{\pm 0.04}$ & $\underline{0.61}^{\pm 0.02}$ & $0.60^{\pm 0.01}$ & $0.58^{\pm 0.03}$ & $\mathbf{0.62}^{\pm 0.01}$ \\
            & DT $\downarrow$ & $0.46^{\pm 0.03}$ & $0.10^{\pm 0.03}$ & $0.12^{\pm 0.03}$ & $\underline{0.45}^{\pm 0.02}$ & $\mathbf{0.32}^{\pm 0.01}$ \\
            & TPR $\uparrow$ & $0.58^{\pm 0.09}$ & $0.99^{\pm 0.01}$ & $1.00^{\pm 0.00}$ & $0.76^{\pm 0.04}$ & $0.94^{\pm 0.02}$ \\
            & TNR $\uparrow$ & $0.78^{\pm 0.03}$ & $0.33^{\pm 0.01}$ & $0.33^{\pm 0.00}$ & $0.73^{\pm 0.05}$ & $0.60^{\pm 0.03}$ \\
        \midrule
        \multirow{5}{*}{Lighting}
            & Acc. $\uparrow$ & $0.69^{\pm 0.04}$ & $0.64^{\pm 0.06}$ & $0.64^{\pm 0.05}$ & $\underline{0.76}^{\pm 0.02}$ & $\mathbf{0.80}^{\pm 0.02}$ \\
            & TWA $\uparrow$ & $\underline{0.58}^{\pm 0.03}$ & $0.55^{\pm 0.03}$ & $0.55^{\pm 0.03}$ & $0.56^{\pm 0.02}$ & $\mathbf{0.63}^{\pm 0.03}$ \\
            & DT $\downarrow$ & $\underline{0.40}^{\pm 0.07}$ & $0.17^{\pm 0.06}$ & $0.18^{\pm 0.03}$ & $0.49^{\pm 0.09}$ & $\mathbf{0.37}^{\pm 0.04}$ \\
            & TPR $\uparrow$ & $0.59^{\pm 0.12}$ & $1.00^{\pm 0.00}$ & $1.00^{\pm 0.00}$ & $0.84^{\pm 0.09}$ & $0.94^{\pm 0.02}$ \\
            & TNR $\uparrow$ & $0.79^{\pm 0.03}$ & $0.27^{\pm 0.13}$ & $0.29^{\pm 0.10}$ & $0.68^{\pm 0.12}$ & $0.66^{\pm 0.05}$ \\
        \midrule
        \multirow{5}{*}{Layout}
            & Acc. $\uparrow$ & $0.70^{\pm 0.02}$ & $\underline{0.74}^{\pm 0.01}$ & $0.71^{\pm 0.02}$ & $\underline{0.74}^{\pm 0.04}$ & $\mathbf{0.81}^{\pm 0.02}$ \\
            & TWA $\uparrow$ & $0.58^{\pm 0.02}$ & $\underline{0.60}^{\pm 0.01}$ & $0.58^{\pm 0.03}$ & $0.57^{\pm 0.02}$ & $\mathbf{0.62}^{\pm 0.02}$ \\
            & DT $\downarrow$ & $0.40^{\pm 0.04}$ & $\underline{0.28}^{\pm 0.02}$ & $\mathbf{0.26}^{\pm 0.02}$ & $0.54^{\pm 0.06}$ & $0.48^{\pm 0.03}$ \\
            & TPR $\uparrow$ & $0.63^{\pm 0.06}$ & $0.98^{\pm 0.03}$ & $0.98^{\pm 0.02}$ & $0.65^{\pm 0.11}$ & $0.86^{\pm 0.06}$ \\
            & TNR $\uparrow$ & $0.77^{\pm 0.04}$ & $0.50^{\pm 0.04}$ & $0.43^{\pm 0.03}$ & $0.84^{\pm 0.05}$ & $0.77^{\pm 0.04}$ \\
        \midrule
        \multirow{5}{*}{Average}
            & Acc. $\uparrow$ & $0.70^{\pm 0.03}$ & $0.68^{\pm 0.02}$ & $0.67^{\pm 0.01}$ & $\underline{0.75}^{\pm 0.02}$ & $\mathbf{0.79}^{\pm 0.00}$ \\
            & TWA $\uparrow$ & $0.59^{\pm 0.02}$ & $\underline{0.62}^{\pm 0.01}$ & $\underline{0.62}^{\pm 0.01}$ & $0.59^{\pm 0.02}$ & $\mathbf{0.65}^{\pm 0.00}$ \\
            & DT $\downarrow$ & $\underline{0.35}^{\pm 0.01}$ & $0.13^{\pm 0.02}$ & $0.11^{\pm 0.00}$ & $0.44^{\pm 0.07}$ & $\mathbf{0.32}^{\pm 0.02}$ \\
            & TPR $\uparrow$ & $0.65^{\pm 0.09}$ & $0.99^{\pm 0.01}$ & $0.99^{\pm 0.00}$ & $0.75^{\pm 0.08}$ & $0.92^{\pm 0.02}$ \\
            & TNR $\uparrow$ & $0.75^{\pm 0.02}$ & $0.37^{\pm 0.04}$ & $0.35^{\pm 0.02}$ & $0.75^{\pm 0.06}$ & $0.66^{\pm 0.03}$ \\
        \bottomrule
    \end{tabular}
    \label{tab:failure_prediction_liberoplus}
\end{table}

\begin{table}[!htb]
    \centering
    \small
    \caption{\textbf{Failure detection results in the real world.} We report the mean and standard deviation over three draws of calibration rollouts.}
    \begin{tabular}{l l c c c c c}
        \toprule
        Task & Metric & STAC & logpZO & RND-OE & ACE & \textbf{VFD (ours)} \\
        \midrule
        \multirow{5}{*}{\texttt{Place}}
            & Acc. $\uparrow$ & $0.42^{\pm 0.02}$ & $\underline{0.71}^{\pm 0.00}$ & $0.63^{\pm 0.02}$ & $0.54^{\pm 0.05}$ & $\mathbf{0.78}^{\pm 0.02}$ \\
            & TWA $\uparrow$ & $0.39^{\pm 0.01}$ & $\mathbf{0.65}^{\pm 0.00}$ & $\underline{0.57}^{\pm 0.02}$ & $0.51^{\pm 0.04}$ & $\underline{0.57}^{\pm 0.01}$ \\
            & DT $\downarrow$ & $0.80^{\pm 0.26}$ & $\mathbf{0.12}^{\pm 0.00}$ & $0.13^{\pm 0.01}$ & $0.32^{\pm 0.07}$ & $\underline{0.50}^{\pm 0.05}$ \\
            & TPR $\uparrow$ & $0.10^{\pm 0.08}$ & $0.98^{\pm 0.03}$ & $1.00^{\pm 0.00}$ & $0.21^{\pm 0.12}$ & $0.86^{\pm 0.00}$ \\
            & TNR $\uparrow$ & $0.75^{\pm 0.04}$ & $0.43^{\pm 0.02}$ & $0.27^{\pm 0.05}$ & $0.88^{\pm 0.02}$ & $0.71^{\pm 0.03}$ \\
        \midrule
        \multirow{5}{*}{\texttt{Stack}}
            & Acc. $\uparrow$ & $0.43^{\pm 0.02}$ & $0.68^{\pm 0.10}$ & $\underline{0.77}^{\pm 0.05}$ & $0.58^{\pm 0.03}$ & $\mathbf{0.81}^{\pm 0.01}$ \\
            & TWA $\uparrow$ & $0.43^{\pm 0.02}$ & $\underline{0.66}^{\pm 0.10}$ & $\mathbf{0.71}^{\pm 0.05}$ & $0.53^{\pm 0.01}$ & $0.64^{\pm 0.02}$ \\
            & DT $\downarrow$ & $1.00^{\pm 0.00}$ & $0.03^{\pm 0.00}$ & $\mathbf{0.12}^{\pm 0.02}$ & $0.26^{\pm 0.07}$ & $\underline{0.54}^{\pm 0.07}$ \\
            & TPR $\uparrow$ & $0.00^{\pm 0.00}$ & $1.00^{\pm 0.00}$ & $1.00^{\pm 0.00}$ & $0.38^{\pm 0.13}$ & $0.64^{\pm 0.05}$ \\
            & TNR $\uparrow$ & $0.85^{\pm 0.04}$ & $0.36^{\pm 0.20}$ & $0.53^{\pm 0.11}$ & $0.79^{\pm 0.08}$ & $0.99^{\pm 0.02}$ \\
        \midrule
        \multirow{5}{*}{\texttt{Drawer}}
            & Acc. $\uparrow$ & $0.45^{\pm 0.04}$ & $\mathbf{0.71}^{\pm 0.02}$ & $\underline{0.67}^{\pm 0.05}$ & $0.48^{\pm 0.04}$ & $0.60^{\pm 0.01}$ \\
            & TWA $\uparrow$ & $0.44^{\pm 0.05}$ & $\mathbf{0.61}^{\pm 0.01}$ & $\underline{0.56}^{\pm 0.01}$ & $0.46^{\pm 0.05}$ & $0.53^{\pm 0.01}$ \\
            & DT $\downarrow$ & $0.86^{\pm 0.24}$ & $\mathbf{0.26}^{\pm 0.04}$ & $0.24^{\pm 0.15}$ & $0.49^{\pm 0.09}$ & $0.49^{\pm 0.00}$ \\
            & TPR $\uparrow$ & $0.03^{\pm 0.05}$ & $0.72^{\pm 0.05}$ & $0.94^{\pm 0.10}$ & $0.09^{\pm 0.08}$ & $0.30^{\pm 0.00}$ \\
            & TNR $\uparrow$ & $0.88^{\pm 0.11}$ & $0.69^{\pm 0.09}$ & $0.40^{\pm 0.20}$ & $0.88^{\pm 0.15}$ & $0.90^{\pm 0.02}$ \\
        \midrule
        \multirow{5}{*}{Average}
            & Acc. $\uparrow$ & $0.43^{\pm 0.02}$ & $\underline{0.70}^{\pm 0.03}$ & $0.69^{\pm 0.04}$ & $0.54^{\pm 0.01}$ & $\mathbf{0.73}^{\pm 0.01}$ \\
            & TWA $\uparrow$ & $0.42^{\pm 0.02}$ & $\mathbf{0.64}^{\pm 0.03}$ & $\underline{0.61}^{\pm 0.02}$ & $0.50^{\pm 0.00}$ & $0.58^{\pm 0.00}$ \\
            & DT $\downarrow$ & $0.89^{\pm 0.08}$ & $\mathbf{0.14}^{\pm 0.01}$ & $0.16^{\pm 0.05}$ & $0.36^{\pm 0.07}$ & $\underline{0.51}^{\pm 0.03}$ \\
            & TPR $\uparrow$ & $0.04^{\pm 0.02}$ & $0.90^{\pm 0.02}$ & $0.98^{\pm 0.03}$ & $0.22^{\pm 0.05}$ & $0.60^{\pm 0.02}$ \\
            & TNR $\uparrow$ & $0.83^{\pm 0.05}$ & $0.49^{\pm 0.04}$ & $0.40^{\pm 0.10}$ & $0.85^{\pm 0.05}$ & $0.87^{\pm 0.02}$ \\
        \bottomrule
    \end{tabular}
    \label{tab:failure_prediction_real_world}
\end{table}

\begin{figure}[tb!]
    \centering
    \includegraphics[width=0.8\linewidth]{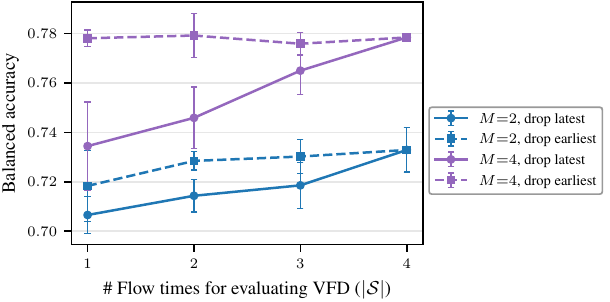}
    \caption{\textbf{Impact of the flow times for evaluating VFD}. Dropping later (larger) flow timesteps from the VFD computation~\eqref{eq:vfd_failure_detection} reduces performance more than dropping earlier ones, indicating that disagreements in the velocity fields are more indicative of epistemic uncertainty when the action chunks contain less noise.}
    \label{fig:vfd_flow_times}
\end{figure}
\paragraph{VFD evaluation flow times.}
We also examine the impact of evaluating the VFD score~\eqref{eq:vfd_failure_detection} not at every ODE state, but only at a subset of evaluation flow times~$\mathcal{S}\subseteq\{0.2, 0.4, 0.6, 0.8\}$.
As shown in~\Cref{fig:vfd_flow_times}, later timesteps where the action chunks contain less noise are more informative for estimating epistemic uncertainty.
For a large four-member ensemble, restricting the computation of VFD to a single flow timestep~$s=0.8$ does not reduce balanced accuracy.
Together with the computational cost analysis in~\Cref{tab:computational_cost}, this experiment indicates that VFD can be computed efficiently while still being a capable estimator of epistemic uncertainty.

\paragraph{Rollout Visualization.}
Figures~\ref{fig:failure_detection_rollout_successes} and~\ref{fig:failure_detection_rollout_fails} visualize exemplary real-world rollouts, with the failure detection scores normalized for readability.
Generally, we observe that the observation-based OOD detectors, logpZO and RND-OE, are prone to raising false alarms (TNR of \num{0.49} and~\num{0.40}, respectively), which happens in both rollouts shown in~\Cref{fig:failure_detection_rollout_successes}.
In contrast, STAC and ACE are unable to detect most failures (TPR of \num{0.04} and~\num{0.22}, respectively), see~\Cref{fig:failure_detection_rollout_fails}.
Compared to the baselines, VFD offers a middle ground, detecting~\qty{60}{\percent} of all failures while flagging only~\qty{13}{\percent} of all successful rollouts as failures and achieving the highest overall balanced accuracy of~\qty{73}{\percent}.

\begin{figure}[!htb]
    \centering
    \includegraphics[width=\linewidth]{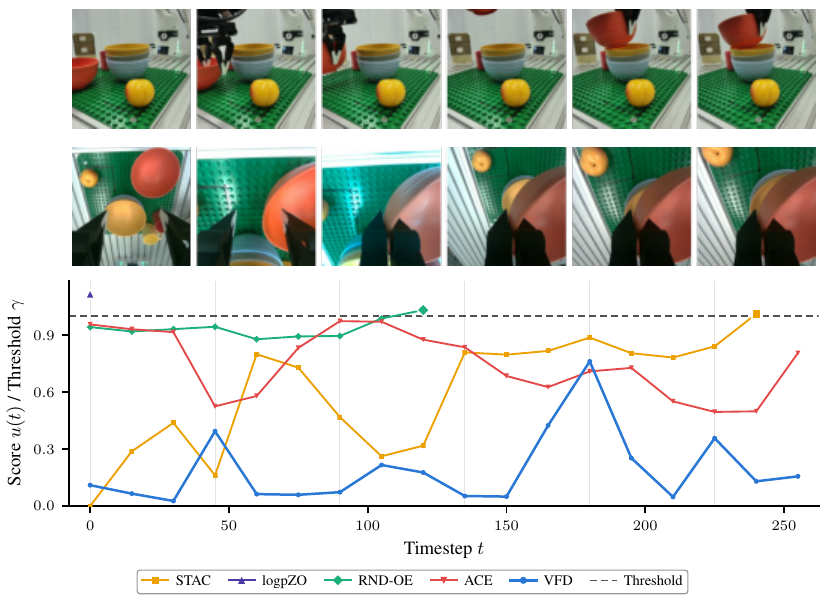} \\
    \includegraphics[width=\linewidth]{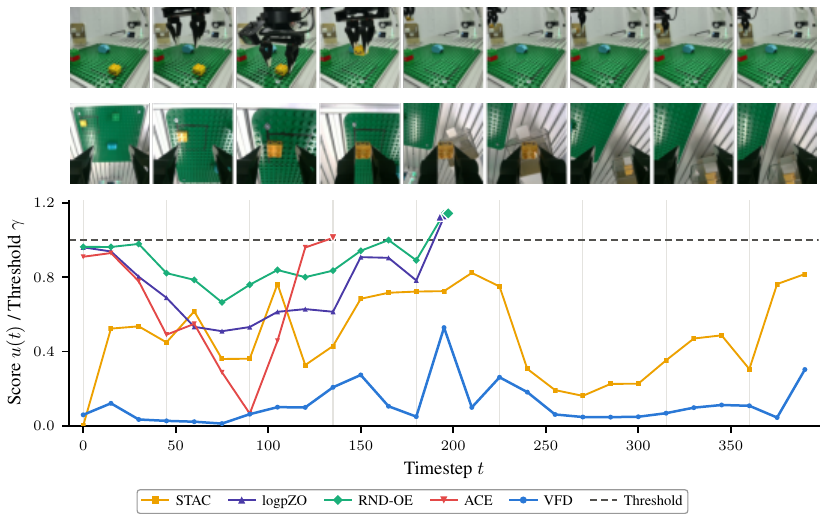}
    \caption{Exemplary successful rollouts that are correctly not flagged as failure by VFD.}
    \label{fig:failure_detection_rollout_successes}
\end{figure}

\begin{figure}[!htb]
    \centering
    \includegraphics[width=\linewidth]{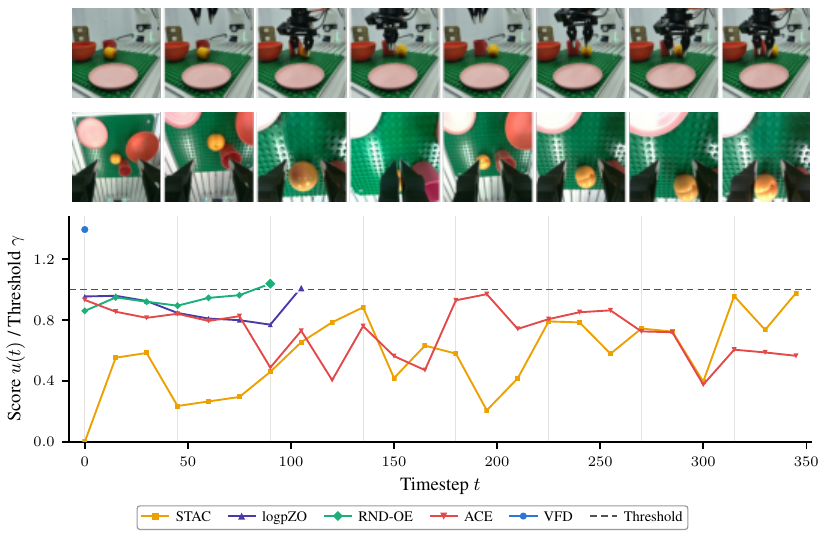} \\
    \includegraphics[width=\linewidth]{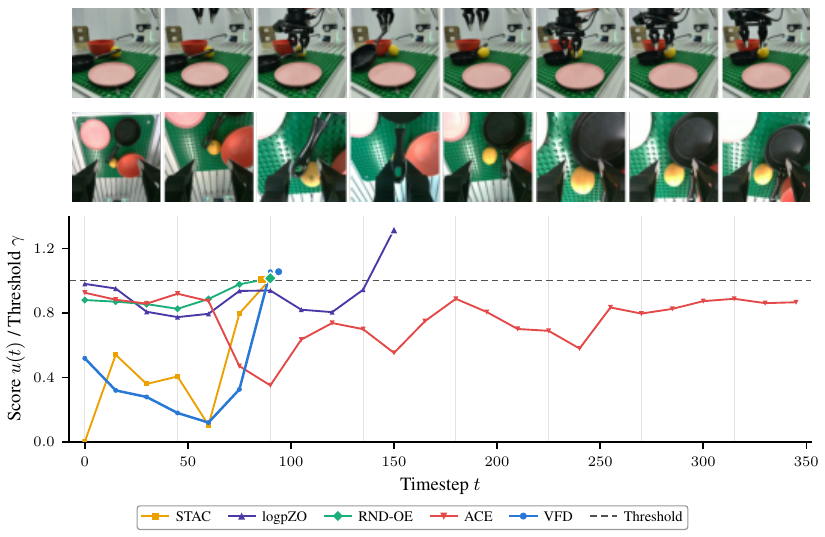}
    \caption{Exemplary failed rollouts that are correctly flagged as failure by VFD.}
    \label{fig:failure_detection_rollout_fails}
\end{figure}

\fi

\ifpaper
\else
\newpage

\bibliographystyle{iclr2027_conference}
\bibliography{ref}

\fi


\end{document}